\documentclass[10pt,twocolumn,letterpaper]{article}

\usepackage[sort, numbers, compress]{natbib}
\usepackage{cvpr}

\cvprfinalcopy 

\usepackage[utf8]{inputenc} 
\usepackage[T1]{fontenc}    
\usepackage{url}            
\usepackage{booktabs}       
\usepackage{amsfonts}       
\usepackage{nicefrac}       
\usepackage{microtype}      
\usepackage{graphicx}  
\usepackage{times} 
\usepackage{amsmath}  
\usepackage{amssymb}
\usepackage[ruled]{algorithm2e}
\usepackage{epsfig} 
\usepackage{epstopdf}
\usepackage{multirow} 
\usepackage{subfigure} 
\usepackage{caption} 
\usepackage{lipsum} 
\usepackage{xcolor} 
\usepackage{bbm}
\usepackage{enumitem}
\usepackage{amsmath}        
\usepackage{amsthm}        
\usepackage{multirow}
\usepackage{mathrsfs}
\usepackage[normalem]{ulem}   

\usepackage[breaklinks=true,bookmarks=false,colorlinks,linkcolor={blue},citecolor={blue},urlcolor={red}]{hyperref}

\newtheorem{theorem}{Theorem}

 \newtheorem*{theorem*}{\protect\theoremname}
 \providecommand{\theoremname}{Theorem}

  \theoremstyle{remark}
  
  \newtheorem*{rem*}{\protect\remarkname}
  
  \newtheorem{lem}[theorem]{Lemma}
  
  \theoremstyle{definition}
  
 \theoremstyle{definition}
 \newtheorem*{defn*}{\protect\definitionname}  
  
\theoremstyle{plain}
\newtheorem*{cor*}{\protect\corollaryname}
\providecommand{\corollaryname}{Corollary}
  \providecommand{\remarkname}{Remark}
  \providecommand{\definitionname}{Definition}
 
\newcommand{\cmtt}[1]{{\fontfamily{cmtt}\selectfont #1}}

\def\cvprPaperID{8452} 

\title{
Task Agnostic Robust Learning on Corrupt Outputs \\
by Correlation-Guided Mixture Density Networks
}

\author{Sungjoon Choi\\
Kakao Brain\\
{\small \cmtt{sam.choi@kakaobrain.com}}
\and
Sanghoon Hong\\
Kakao Brain\\
{\small \cmtt{sanghoon.hong@kakaobrain.com}}
\and
Kyungjae Lee\\
Seoul National University\\
{\small \cmtt{kyungjae.lee@rllab.snu.ac.kr}}
\and
Sungbin Lim\thanks{This work is done at Kakao Brain}\\
UNIST\\
{\small \cmtt{sungbin@unist.ac.kr}}
}

\begin{document}
\maketitle

%
%
\begin{abstract}
In this paper, we focus on weakly supervised learning
with noisy training data for both classification
and regression problems.
We assume that the training outputs are collected from a mixture of
a target and correlated noise distributions.
Our proposed method
simultaneously estimates the target distribution and 
the quality of each data
which is defined as the correlation between
the target and data generating distributions.
The cornerstone of the proposed method is
a Cholesky Block
that enables modeling dependencies among mixture 
distributions in a differentiable manner
where we maintain the distribution over the network weights.
We first provide illustrative examples in both regression and classification
tasks to show the effectiveness of the proposed method.
Then, the proposed method is extensively evaluated in a number of experiments
where we show that it constantly shows
comparable or superior performances
compared to existing baseline methods in the handling of noisy data.
\end{abstract}

%
%
\section{Introduction}

Training a deep neural network requires 
immense amounts of training data
which are often collected using 
crowd-sourcing methods,
such as Amazon’s Mechanical Turk (AMT). 
However, in practice, the crowd-sourced labels are often noisy
\cite{Bi_14_Crowdsource}.
Furthermore, naive training of deep neural networks is often vulnerable to
over-fitting given the noisy training data 
in that they are capable of memorizing
the entire dataset even with inconsistent labels,
leading to a poor generalization performance
\cite{Zhang_17_rethinking}. 
We address this problem through the principled idea of
\emph{revealing the correlations} between the target distribution
and the other (possibly noise) distributions
by assuming that a training dataset is sampled from 
a mixture of a target distribution 
and other correlated distributions.

Throughout this paper,
we aim to address the following two questions:
1) How can we define (or measure) the quality of training data in a principled manner?
2) In the presence of inconsistent outputs, how can we infer the target distribution
in a scalable manner?
Traditionally, noisy outputs are handled by modeling additive random distributions,
often leading to robust loss functions \cite{Hampel_11_Robust} or
estimating the structures of label corruptions in classifications tasks 
\cite{Jindal_16_noisyLabel} (for more details, refer to Section \ref{sec:rel}).

To address the first question,
we leverage the concept of a correlation.
Precisely, we define and measure 
the quality of training data
using the correlation between the target distribution
and the data generating distribution. 
However, estimating the correct correlation 
requires access to the target distribution,
whereas learning the correct target distribution 
requires knowing the correlation between
the distributions to be known,
making it a chicken-and-egg problem. 
To address the second question,
we present a novel method
that simultaneously estimates the target distribution
as well as the correlation
in a fully differentiable manner
using stochastic gradient descent methods.

The cornerstone of the proposed method is
a \emph{Cholesky Block}
in which we employ the Cholesky transform
for sampling the weights of a neural network
that enables us to model correlated outputs.
Similar to Bayesian neural networks \cite{Blundell_15},
we maintain the probability distributions over
the weights, but we also leverage mixture distributions
to handle the inconsistencies in a dataset.
To the best of our knowledge, 
this is the first approach
simultaneously to infer the target distribution and 
the output correlations using a neural network
in an end-to-end manner.
We will refer to this framework as \emph{ChoiceNet}.

\emph{ChoiceNet} is first applied to synthetic regression tasks
and a real-world regression task
where we demonstrate its robustness to extreme outliers
and its ability to distinguish the target distribution
and noise distributions. 
Subsequently, we move on to image classification tasks
using a number of benchmark datasets
where we show that it shows comparable or superior performances
compared to existing baseline methods in terms of robustness
with regard to handling different types of noisy labels.

%
%
\section{Related Work}
\label{sec:rel}

Recently, robustness in deep learning has been 
actively studied \cite{Fawzi_17_robustness}
as deep neural networks are being applied to 
diverse tasks involving real-world applications such as 
autonomous driving \cite{Paden_16} 
or medical diagnosis \cite{Gulshan_16_MedicalDeepLearning}
where a simple malfunction can have catastrophic results
\cite{Tesla_16}. 

Existing work for handing noisy training data can be categorized into 
four groups:
small-loss tricks 
\cite{Jiang_17_mentornet,Ren_18,Han_18, Malach_17_Decoupling},
estimating label corruptions 
\cite{Patrini_17_LossCorrection,Goldberger_17_NoiseAdaptation,
Sukhbaatar_14, Bekker_16_unreliableLabel, Hendrycks_18_GLC, Veit_17_noisy},
using robust loss functions 
\cite{Natarajan_13_NoisyLabel, Belagiannis_15_RobustReg},
and explicit and implicit regularization methods 
\cite{Reed_14_bootstrap,Lee_13_pseudo, Goodfellow_16_DLbook, Xie_16_disturblabel, Tokozume_18_btwClassEx,Zhang_18_mixup,Miyato18,Tarvainen_17,Laine_17}.
Our proposed method is mostly related to the robust loss function approach but 
cannot fully be categorized into this group in that
we present a novel architecture, a mixture of correlated densities network block,
for achieving robustness based on the correlation estimation. 

First of all, the small-loss tricks selectively focus on training instances based on 
a certain criterion, such as having small cost values  \cite{Han_18}. 
\cite{Malach_17_Decoupling} proposed a meta-algorithm
for tackling the noisy label problem by training 
two networks only when the predictions 
of the two networks disagree, where
selecting a proper network from among the two networks
can be done using an additional clean dataset.
\cite{Ren_18} reweighs the weight of each training instance 
using a small amount of clean validation data. 
MentorNet \cite{Jiang_17_mentornet} 
concentrated on the training of an additional neural network,
which assigns a weight to each instance of training data
to supervise the training of a base network,
termed StudentNet,
to overcome the over-fitting of corrupt training data.
Recent work \cite{Han_18} presented Co-teaching
by maintaining two separate networks where each 
network is trained with small-loss instances selected
from its peer network. 

The second group of estimating label corruption information
is mainly presented for classification tasks where
training labels are assumed to be corrupt 
with a possibly unknown corruption matrix. 
An earlier study in \cite{Bekker_16_unreliableLabel}
proposed an extra layer for the modeling of output noises.
\cite{Jindal_16_noisyLabel} extended 
the approach mentioned above
by adding an additional noise adaptation layer with
aggressive dropout regularization. 
A similar method was then proposed in
\cite{Patrini_17_LossCorrection}
which initially estimated 
the label corruption matrix with 
a learned classifier and used the corruption matrix 
to fine-tune the classifier. 
Other researchers
\cite{Goldberger_17_NoiseAdaptation} 
presented a robust training method
that mimics the EM algorithm to train a neural network,
with the label noise modeled as
an additional softmax layer, similar to earlier work
\cite{Jindal_16_noisyLabel}. 
A self-error-correcting network was also presented
\cite{Liu_17_SelfCorrect}.
It switches the training labels based on the learned model 
at the beginning stages by assuming that the deep model 
is more accurate during the earlier stage of training.

Researchers have also focussed on using robust loss functions;
\cite{Natarajan_13_NoisyLabel} studied the problem of binary classification 
in the presence of random labels and 
presented a robust surrogate loss function for handling noisy labels.
Existing loss functions for classification were studied 
\cite{Ghosh_17_robustLoss}, with the results showing that 
the mean absolute value of error is inherently robust to label noise. 
In other work \cite{Belagiannis_15_RobustReg}, a robust loss function for
deep regression tasks was proposed using Tukey's biweight function with
the median absolute deviation of the residuals.

The last group focusses on using implicit or
explicit regularization methods while training.
Adding small label noises while training 
is known to be beneficial to training, as it can be regarded as 
an effective regularization method
\cite{Lee_13_pseudo, Goodfellow_16_DLbook}.
Similar methods have been proposed to tackle 
noisy outputs. 
A bootstrapping method
\cite{Reed_14_bootstrap}
which trains a neural network with a convex combination of 
the output of the current network and the noisy target was proposed. 
\cite{Xie_16_disturblabel} proposed DisturbLabel, 
a simple method that randomly replaces a percentage of the labels 
with incorrect values for each iteration.
Mixing both input and output data was also proposed 
\cite{Tokozume_18_btwClassEx,Zhang_18_mixup}. 
One study
\cite{Zhang_18_mixup} considered the
image recognition problem under label noise 
and the other \cite{Tokozume_18_btwClassEx}
focused on a sound recognition problem. 
The temporal ensemble was proposed in \cite{Laine_17} where 
an unsupervised loss term of fitting the output of an augmented input
to the augmented target updated with an exponential moving average. 
\cite{Tarvainen_17} extends the temporal ensemble in \cite{Laine_17}
by introducing a consistency cost function that minimizes 
the distance between the weights of 
the student model and the teacher model. 
\cite{Miyato18} presented a new regularization method 
based on virtual adversarial loss which measures the smoothness of
conditional label distribution given input. 
Minimizing the virtual adversarial loss has a regularizing effect in that
it makes the model smooth at each data point.

The foundation of the proposed method is
a \emph{Cholesky Block}
where the output distribution is modeled using 
a mixture of correlated distributions. 
Modeling correlations of output training data has been 
actively studied in light of Gaussian processes
\cite{Rasmussen_06}.
MTGPP \cite{Bonilla_08_MTGPP}
that models the correlations of multiple tasks
via Gaussian process regression was proposed
in the context of a multi-task setting. 
\cite{SJChoi_16} proposed a robust learning from demonstration method
using a sparse constrained leverage optimization method
which estimates the correlation between the training outputs
and showed its robustness compared to several baselines.

\begin{figure*}[!t] 
	\centering 
	\includegraphics[width=0.73\textwidth]
		{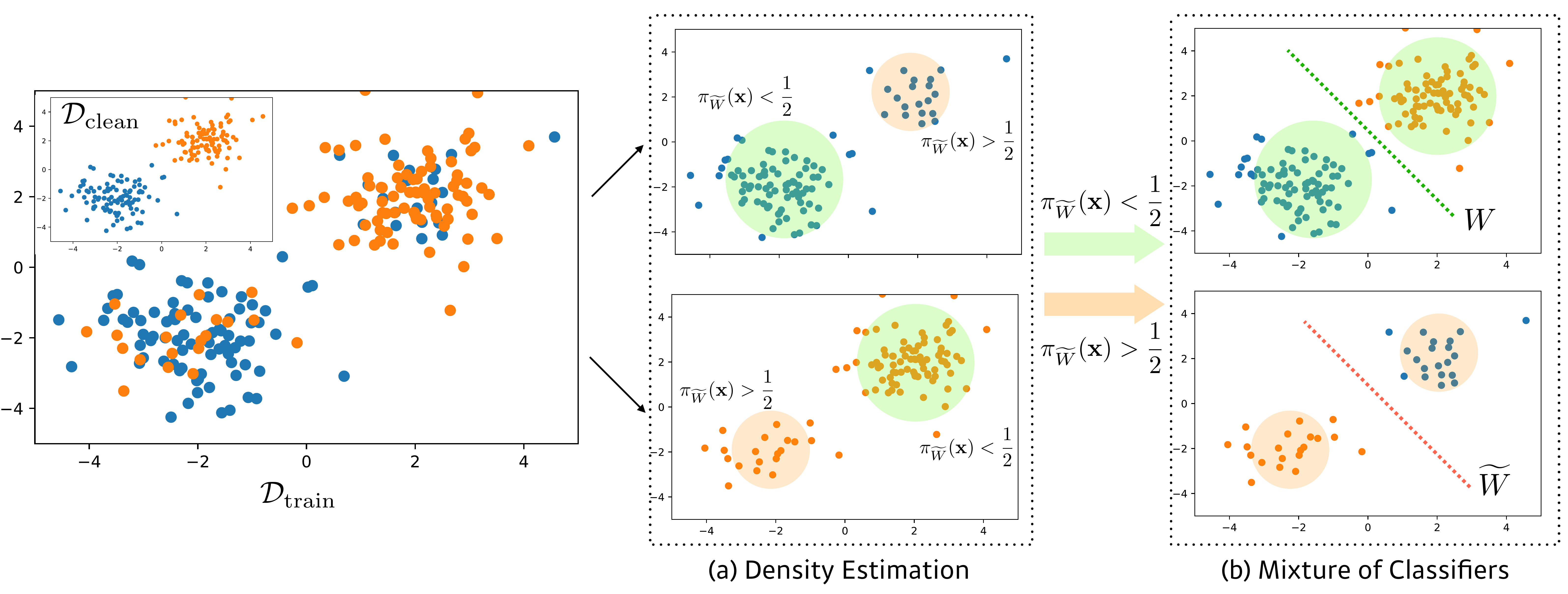}
	\caption{
	    A process of binary classification on corrupt data using the mixture of (a) densities and (b) classifiers through \eqref{eq:mixture-loss}. 
	    $\pi_{\widetilde{W}}(\mathbf{x})$ is the mixture weight which models the choice probability of $\widetilde{W}$ and estimates the corruption probability $\pi_{\text{corrupt}}$ in \eqref{eq:data-generating}. }
	   	\label{fig:mixture}
\end{figure*}

%
%
\section{Proposed Method}
\label{sec:prop}

In this section, we present the motivation
and the model architecture of the proposed method. 
The main ingredient is a \emph{Cholesky Block}
which can be built on top of any arbitrary base network.
First, we illustrate the motivations of 
designing the \emph{Cholesky Block} 
in Section \ref{subsec:motivation}.
Section \ref{subsec:sample_corr} introduces a Cholesky transform 
which enables correlated sampling procedures.
Subsequently, we present the overall mechanism of the proposed method 
and its loss functions for regression and classification tasks 
in Section \ref{subsec:model}
followed by
illustrative synthetic examples
in Section \ref{subsec:ill}.

\subsection{Motivation}
\label{subsec:motivation}


Denote training data with correct (clean) labels 
by $\mathcal{D}_{\text{clean}}$
whose elements $(\mathbf{x},y)\in\mathcal{D}_{\text{clean}}$ are
determined by a relation $y=f(\mathbf{x})$ for a regression task
and $y\in L$ for a classification task where $L$ is a discrete set.
In this paper,
we assume that the corrupt training data 
$(\mathbf{x},\hat{y})\in\mathcal{D}_{\text{train}}$
is generated by

\textbf{Regression}:
\begin{align}
    \label{eq:regression}
	\hat{y}=\begin{cases}
	f(\mathbf{x})+\epsilon & \text{with }1-p\\
	g(\mathbf{x})+\xi & \text{with }p
	\end{cases}
\end{align}

\textbf{Classification}:
\begin{align}
    \label{eq:classification}
	\hat{y}=\begin{cases}
	y & \text{with }1-p\\
	L \setminus\{y\} & \text{with }p
	\end{cases}
\end{align}
where $g$ is an arbitrary function. 
Here $\epsilon$ and $\xi$ are heteroscedastic additive noises
and $p$ indicates the corruption probability. 
Then, the above settings naturally employ
the mixture of the conditional distributions:
\begin{equation}
    P(\hat{y}|\mathbf{x})
        =(1-\pi_{\text{corrupt}})P_{\text{target}}(\hat{y}|\mathbf{x})
        +\pi_{\text{corrupt}}P_{\text{noise}}(\hat{y}|\mathbf{x})
    \label{eq:data-generating}
\end{equation}
where $\pi_{\text{corrupt}}$ models the corrupt ratio $p$.
In particular, we model the target conditional density 
$P_{\text{target}}(\cdot|\cdot)$
using a parametrized distribution with
a Gaussian distribution with a fixed variance $\hat{\sigma}^{2}$,
i.e., 
$P(\cdot|\mathbf{x})=\mathcal{N}(f_{\mathbf{\theta}}(\mathbf{x}),\,\hat{\sigma}^{2})$
where $f_{\theta}(\cdot)$ 
can be any functions, e.g., a feed-forward network,
parametrized with $\theta$.

While it may look similar to
a Huber’s $\epsilon$-contamination model \cite{Huber_11},
one major difference is that,
our method quantifies the irrelevance (or independence)
of noise distributions by utilizing 
the input-dependent correlation $\rho(\mathbf{x})$ between 
$P_{\text{target}}(\cdot|\mathbf{x})$ and $P_{\text{noise}}(\cdot|\mathbf{x})$.
To be specific, $P_{\text{noise}}(\cdot|\mathbf{x})$ will be a function of
$P_{\text{target}}(\cdot|\mathbf{x})$ and $\rho(\mathbf{x})$.

In the training phase, we jointly optimize 
$P_{\text{target}}(\cdot|\mathbf{x})$ and $\rho(\mathbf{x})$.
Intuitively speaking, irrelevant noisy data will be modeled to be
collected from a class of $P_{\text{noise}}$ with relatively small or negative $\rho$
which can be shown in Figure \ref{fig:corr_check}.
Since we assume that the correlation (quality) is not explicitly given, 
we model the $\rho$ of each data to be a function of an input $\mathbf{x}$
i.e., $\rho_{\phi}(\mathbf{x})$, parametrized by $\phi$ and jointly
optimize $\phi$ and $\theta$.


%
%
\paragraph{Why correlation matters in mixture modeling?}

Let us suppose a binary classification problem 
$L = \{-1,1\}$ and
a feature map $h:\mathcal{X}\to\mathcal{F}$ and 
a (stochastic)
linear functional $W:\mathcal{F}\to\mathbb{R}$ are given. 
For $(\mathbf{x},y)\in\mathcal{D}_{\text{clean}}$, 
we expect that $W$ and $h$ will be optimized as follows:
\[
Wh(\mathbf{x})\sim\begin{cases}
\mathcal{N}(\mu_{+},\sigma^{2}) & :y=1\\
\mathcal{N}(\mu_{-},\sigma^{2}) & :y=-1
\end{cases},\quad\mu_{+}>0,\mu_{-}<0
\]
so that $-Wh(\mathbf{x)}\cdot y<0$ holds.
However, if corrupt training data 
are given by \eqref{eq:classification}, 
the linear functional $W$ and the feature map $h(\cdot)$ 
may have $-Wh(\mathbf{x)}\cdot\hat{y}>0$
and using an ordinary loss function, such as a cross-entropy loss,
might lead to over-fitting of the contaminated pattern 
of $\mathcal{D}_{\text{train}}$.
Motivated from (\ref{eq:data-generating}),
we employ the mixture density to discriminate the corrupt
data by using another linear classifier $\widetilde{W}$ 
which is expected to reveal the reverse patterns
by minimizing the following mixture classification loss:
\begin{equation}
-\left\{ (1-\pi_{\widetilde{W}}(\mathbf{x}))Wh(\mathbf{x})-\pi_{\widetilde{W}}(\mathbf{x})\widetilde{W}h(\mathbf{x})\right\} \cdot\hat{y}\label{eq:mixture-loss}
\end{equation}
Here $\pi_{\widetilde{W}}(\mathbf{x})$ is the mixture weight which models
the choice probability of the reverse classifier $\widetilde{W}$ 
and eventually estimates the corruption probability $\pi_{\text{corrupt}}$ in \eqref{eq:data-generating}. 
See Figure \ref{fig:mixture} for the illustrative example 
in binary classification.

However, the above setting is not likely to work in practice
as both $W$ and $\widetilde{W}$ may learn
the corrupt patterns independently
hence $\pi_{\widetilde{W}}(\mathbf{x})$ adhere to $1/2$ under \eqref{eq:mixture-loss}.
In other words, the lack of dependencies between $W$ and $\widetilde{W}$
makes it hard to distinguish clean and corrupt patterns.
To aid $W$ to learn the pattern of data with a clean label, 
we need a self-regularizing way
to help $\pi_{\widetilde{W}}$ to infer the corruption probability of given data point $\mathbf{x}$ 
by guiding $\widetilde{W}$ to learn the reverse mapping of given feature $h(\mathbf{x})$.
To resolve this problem, let us consider different linear functional
$\widehat{W}$ with negative correlation with $W$, i.e., $\rho(\widehat{W},W)<0$.
Then this functional maps the feature $h(\mathbf{x})$ as follows: 
\[
\widehat{W}h(\mathbf{x})\sim\begin{cases}
\mathcal{N}(\rho\mu_{+},\sigma^{2}) & :y=1\\
\mathcal{N}(\rho\mu_{-},\sigma^{2}) & :y=-1
\end{cases}
\]
since $\rho(\widehat{W},W)=\rho(\widehat{W}h,Wh)$ so we have $-\widehat{W}h(\mathbf{x})\cdot\hat{y}<0$
if $\hat{y}=-y$. Eventually, \eqref{eq:mixture-loss} is minimized 
when $\pi_{\widehat{W}}(\mathbf{x})\approx1$
if output is corrupted, i.e., $\hat{y}=-y$, and $\pi_{\widetilde{W}}(\mathbf{x})\approx0$
otherwise. In this way, we can make $\pi_{\widehat{W}}(\mathbf{x})$
to learn the corrupt probability for given data point.

We provide illustrative examples in both regression and classification
that can support the aforementioned motivation
in Section \ref{subsec:ill}.

%
%
\subsection{Cholesky Block for Correlated Sampling}
\label{subsec:sample_corr}

Now we introduce a Cholesky transform that enables 
modeling dependencies among output mixtures in a differentiable manner.
The Cholesky transform is a way of constructing a random variable 
which is correlated with other random variables
and can be used as a sampling routine for sampling
correlated matrices from two uncorrelated random matrices.

To be specific,
suppose that $w \sim \mathcal{N}(\mu_w,\,\sigma^2_w)$ and 
$z \sim \mathcal{N}(0,\,\sigma^2_z)$ and our goal is to
construct a random variable $\tilde{w}$ such that
the correlation between $w$ and $\tilde{w}$ becomes $\rho$.
Then, the Cholesky transform which is defined as
\begin{equation}
\label{eq:cholesky}
    \begin{aligned}
&\text{\cmtt{Cholesky}}(w,z,\rho,\mu_{w},\sigma_{w},\sigma_{z})
\\ \quad&:=\rho\mu+\sqrt{1-\rho^{2}}
    \left(\rho\frac{\sigma_{z}}{\sigma_{w}}(w-\mu)+z\sqrt{1-\rho^{2}}\right)
    \end{aligned}
\end{equation}
is a mapping from 
$(w,z)\in\mathbb{R}^{2}$ to $\mathbb{R}$
and can be used to construct $\tilde{w}$.
In fact, $\tilde{w} = \text{\cmtt{Cholesky}}
    (w,z,\rho,\mu_{w},\sigma_{w},\sigma_{z})$ 
and we can easily use (\ref{eq:cholesky}) to 
construct a feed-forward layer 
with a correlated weight matrix
which will be referred to as
a \emph{Cholesky Block} as shown in Figure \ref{fig:cholesky_block}.

We also show that the correlation is preserved through 
the affine transformation
making it applicable to a single fully-connected layer
where all the derivations
and proofs can be found in the supplementary material.
In other words, we model correlated outputs by
first sampling correlated weight matrices using Cholesky transfrom
in an element-wise fashion 
and using the sampled weights for an affine transformation of 
a feature vector of a feed-forward layer. 
One can simply use a reparametrization trick \cite{kingma2017variational}
for implementations.

%
%
\begin{figure}[!t] 
	\centering 
	\includegraphics[width=\columnwidth]
		{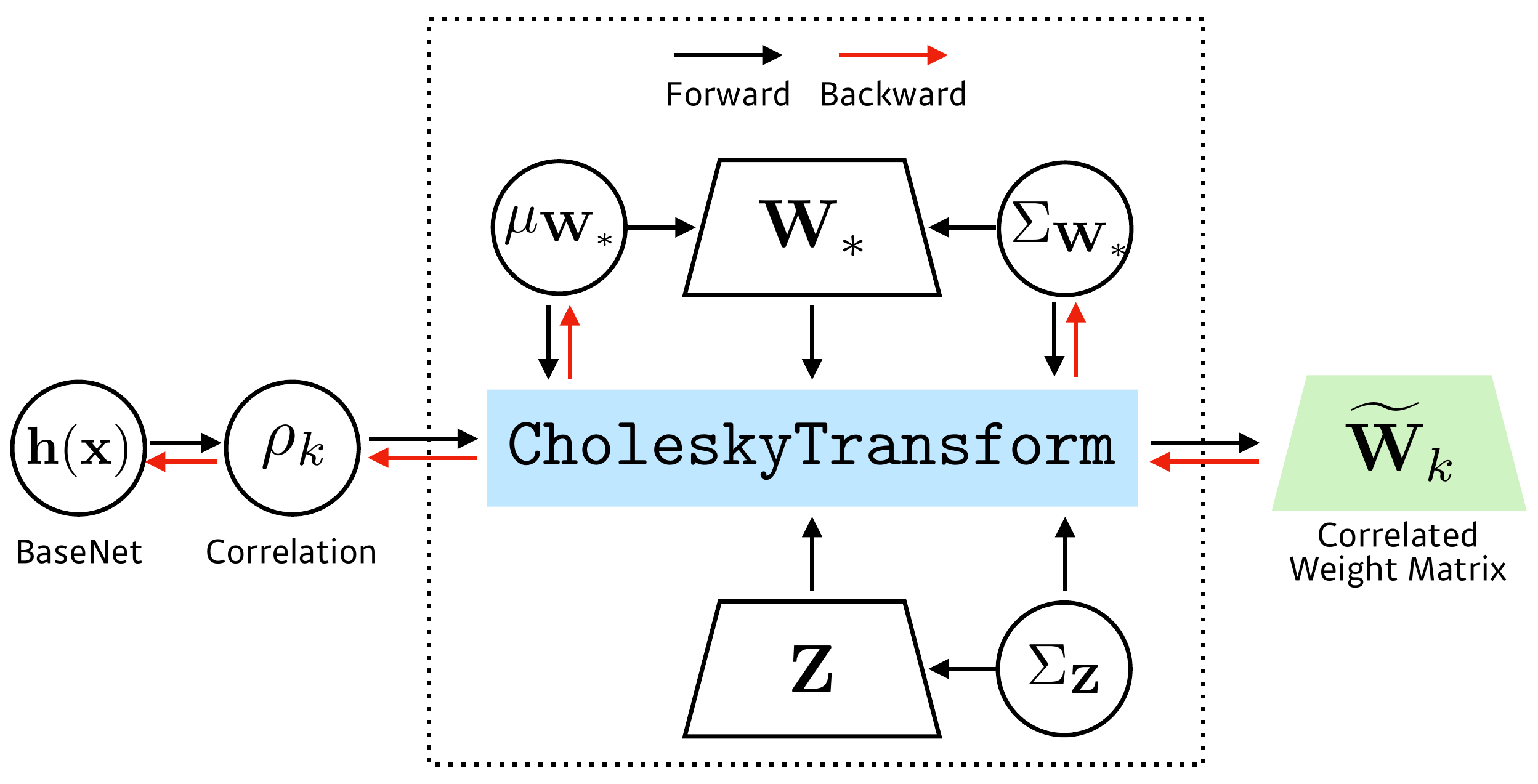}
	\caption{
	    Illustration of a \emph{Cholesky Block}. Every block shares target weight matrix $\mathbf{W}_{*}$ and auxiliary matrix $\mathbf{Z}$, and outputs correlated weight matrix $\widetilde{\mathbf{W}}_{k}$ through \cmtt{CholeskyTransform} (see \eqref{eq:cholesky}) to distinguish the abnormal pattern from normal one which will be learned by $\mathbf{W}_{*}$. 
		}
	\label{fig:cholesky_block}
\end{figure}

%
%
\subsection{Overall Mechanism of ChoiceNet}
\label{subsec:model}

%
%
In this section we describe the model architecture and
the overall mechanism of ChoiceNet.
In the following,
$\tau^{-1} >0$ is a small constant indicating the expected variance
of the target distriubtion.
$\mathbf{W}_{\mathbf{h}\to\boldsymbol{\rho}}$, $\mathbf{W}_{\mathbf{h}\to\boldsymbol{\pi}}\in \mathbb{R}^{K\times Q}$ 
and $\mathbf{W}_{\mathbf{h}\to\boldsymbol{\Sigma_{0}}}\in\mathbb{R}^{D\times Q}$ are weight matrices 
where $Q$ and $D$ denote the dimensions of a feature vector $\mathbf{h}$
and output $\mathbf{y}$, respectively, and $K$ is the number of mixtures\footnote{
Ablation studies regarding changing $K$ and $\tau^{-1}$
are shown in the supplementary material where the results show that
these hyper-parameters
are not sensitive to the overall learning results.}.

ChoiceNet is a twofold architecture: 
(a) an arbitrary base network and 
(b) the \emph{Cholesky Block} (see Figure \ref{fig:cholesky_block}).
Once the base network extracts features,
The \emph{Cholesky Block} estimates the mixture of
the target distribution and 
other correlated distributions using the Cholesky transform 
in \eqref{eq:cholesky}. 
While it may seem to resemble a mixture density network (MDN)
\cite{bishop1994mixture},
this ability to model dependencies between mixtures lets us to 
effectively infer and distinguish the target distributions 
from other noisy distributions as will be shown in the experimental sections.
When using the MDN, particularly, it is not straightforward to select which
mixture to use other than using the one with the largest 
mixture probability which may lead to inferior performance
given noisy datasets\footnote{
We test the MDN in both regression and classification tasks
and the MDN shows poor performances.}.


Let us elaborate on the overall mechanism of ChoiceNet.
Given an input $\mathbf{x}$, a feature vector 
$\mathbf{h} \in \mathbb{R}^Q$ is computed
from any arbitrary mapping such as ResNet \cite{He_16}
for images or word embedding \cite{Levy_14} for natural languages.
Then the network computes $K-1$ correlations,
$\{\rho_1,\rho_2(\mathbf{h}),...,\rho_K(\mathbf{h})\}$, for $K$ mixtures
where the first $\rho_1=1$ is reserved to model the target distribution,
In other words, the first mixture becomes our target distribution
and we use the predictions from the first mixture in the inference phase.

The mean and variance of the weight matrix,
$\mu_{\mathbf{W}} \in \mathbb{R}^{Q \times D}$ and 
$\Sigma_{\mathbf{W}} \in \mathbb{R}^{Q \times D}$,
are defined and updated for modeling correlated output distributions
which is analogous to a Bayesian neural network \cite{Blundell_15},
These matrices can be back-propagated using the reparametrization trick. 
The \emph{Cholesky Block} also computes 
the base output variance $\Sigma_0(\mathbf{h})$ similar to an MDN.

Then we sample $K$ weight matrices $\{ \widetilde{\mathbf{W}}_i \}_{i=1}^K$
from $\{\mu_{*},\Sigma_{*}\}$
and $\{\rho_1,\rho_2(\mathbf{h}),...,\rho_K(\mathbf{h})\}$
using the Cholesky transform \eqref{eq:cholesky}
so that the correlations between $\widetilde{\mathbf{W}}_i$ and $\mathbf{W}_{*}$
becomes $\rho_i(\cdot)$.
Note that the correlation is preserved through an affine transform. 
The $K$ sampled feedforward weights,
$\{ \widetilde{\mathbf{W}}_i \}_{i=1}^K$, are used to compute
$K$ correlated output mean vectors, $\{ \mu_i \}_{i=1}^K$.
Note that the correlation between $\mu_1$ and $\mu_i$ 
also becomes $\rho_i(\cdot)$. 
We would like to emphasize that,
as we employ Gaussian distributions in the Cholesky transform,
the influences of uninformative or independent data,
whose correlations, $\rho$, are close to $0$,
is attenuated as their variances increase
\cite{kendall2017uncertainties}.

\SetKwInOut{Input}{Input}
\SetKwInOut{Output}{Output}
\begin{algorithm}[t]
\Input{$\mathcal{D}_{\text{train}}$, $K$, $\tau$, $\lambda$, $\mathbf{h}:\mathcal{X}\to\mathbb{R}^{Q}$}

Initialize $\mu_{*}, \Sigma_{*}, \Sigma_{\mathbf{Z}} \in\mathbb{R}^{Q\times D}$

$\qquad\qquad\mathbf{W}_{\mathbf{h}\to\boldsymbol{\rho}}, \mathbf{W}_{\mathbf{h}\to\boldsymbol{\pi}},\mathbf{W}_{\mathbf{h}\to\Sigma_{0}} \in\mathbb{R}^{K\times Q}$

\While{True}{
Sample $\mathbf{W}_{*}\sim\mathcal{N}(\mu_{*},\Sigma_{*})$, $\mathbf{Z}\sim\mathcal{N}(\boldsymbol{0},\Sigma_{\mathbf{Z}})$

\For{$k \in \{1,\ldots ,K\}$}{
    $\rho_{k} = \tanh(\mathbf{W}_{\mathbf{h}\to\boldsymbol{\rho}}\mathbf{h})_{k}$ $\qquad(\rho_{1}=1)$ \\
    $\pi_{k} = \text{softmax}(\mathbf{W}_{\mathbf{h}\to\boldsymbol{\pi}}\mathbf{h})_{k}$ \\
    $(\Sigma_{0})_{k}=\exp(\mathbf{W}_{\mathbf{h}\to\Sigma_{0}}\mathbf{h})_{k}$ \\ 
    $\Sigma_{k} = (1-\rho_{k}^{2})(\Sigma_{0})_{k} + \tau^{-1}$ \\
    $\widetilde{\mathbf{W}}_{k}$ = \cmtt{Cholesky}$(\mathbf{W}_{*}, \mathbf{Z},\rho_{k},\mu_{*},\Sigma_{*},\Sigma_{\mathbf{Z}})$ \\
    $\mu_{k} = \widetilde{\mathbf{W}}_{k}\mathbf{h}$
}

Compute $\mathcal{L}(\mathcal{D}_{\text{train}}|(\pi_{k},\mu_{k},\Sigma_{k})_{k=1}^{K})$ \\

Update $\mathbf{h}, \mathbf{W}_{\mathbf{h}\to\boldsymbol{\rho}}, \mathbf{W}_{\mathbf{h}\to\boldsymbol{\pi}},\mathbf{W}_{\mathbf{h}\to\Sigma_{0}},\mu_{*},\Sigma_{*}$

}
\Return{$\mathbf{W}_{*}, \mathbf{h}$}
\caption{ChoiceNet Algorithm}
\label{alg:choice-net}
\end{algorithm}

The output variances of $k$-th mixture is computed from
$\rho_k(\mathbf{h})$ and the based output variance $\Sigma_0(\mathbf{h})$
as follows:
\begin{equation}
	\Sigma_k = (1-\rho^2_k(\mathbf{h})) \Sigma_0(\mathbf{h}) + \tau^{-1} 
	\in \mathbb{R}^D
	\label{eqn:sig_k}
\end{equation}
This has an effect of increasing the variance of the mixture
which is less related (in terms of the absolute value of a correlation)
with the target distribution.
The \emph{Cholesky Block} also computes the mixture probabilities
of $k$-th mixture, $\pi_k(\mathbf{h})$, akin to an MDN.
The overall process of the \emph{Cholesky Block}
is summarized in Algorithm \ref{alg:choice-net}
and Figure \ref{fig:overall}.
Now, let us introduce the loss functions for 
regression and classification tasks.


\paragraph{Regression}

For the regression task, we employ both $L_{2}$-loss 
and the standard MDN loss
(\cite{bishop1994mixture,
choi2017uncertainty, christopher2016pattern})
\begin{equation}
\begin{aligned}
    \label{loss-reg}
    &\mathcal{L}(\mathcal{D}) 
    = \frac{1}{N}\sum_{i=1}^{N} 
        \lambda_1
        \|\mathbf{y}_{i} 
        - \mu_{1}(\mathbf{x}_{i})\|_{2}^{2} \\
         &+ \frac{1}{N}\sum_{i=1}^{N} 
        \lambda_2
        \log\left(\sum_{k=1}^{K}\pi_{k}(\mathbf{x}_{i})
        \mathcal{N}(\mathbf{y}_{i};\mu_{k}(\mathbf{x}_{i}),
        \Sigma_{k}(\mathbf{x}_{i})) \right)
\end{aligned}
\end{equation}
where $\lambda_1$ and $\lambda_2$ are hyper-parameters and $\mathcal{N}(\cdot | \mu, \Sigma)$ is the density of multivariate Gaussian:
\begin{align*}
\mathcal{N}(\mathbf{y}_{i};\mu_{k}, \Sigma_{k})
= \prod_{d=1}^{D}\frac{1}{\sqrt{2\pi \Sigma_{k}^{(d)}}}\exp\left(-\frac{| y_{i}^{(d)} - \mu_{k}^{(d)}|^{2}}{2 \Sigma_{k}^{(d)}}\right)
\end{align*}

%
%
\begin{figure}[!t] 
	\centering 
	\includegraphics[width=\columnwidth]
		{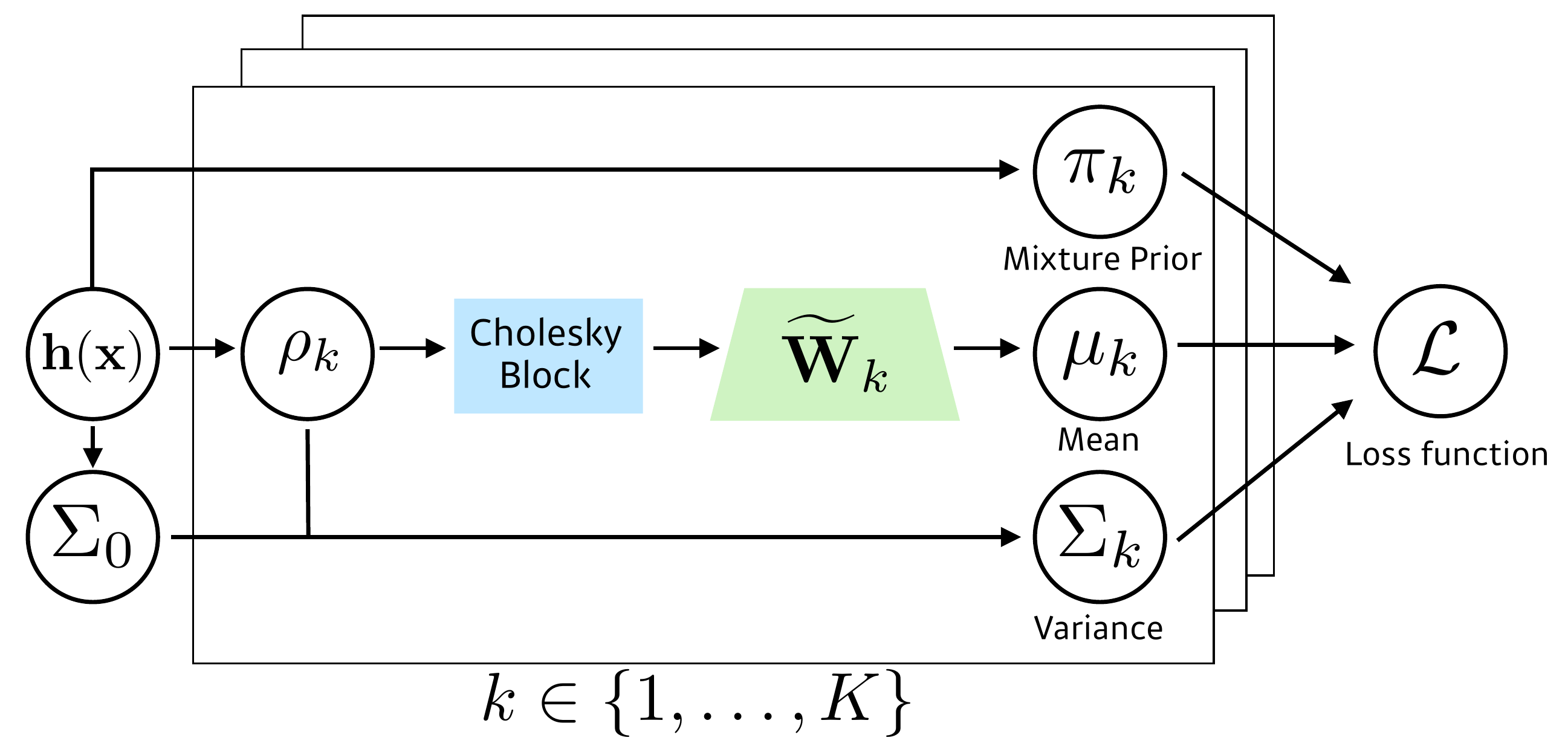}
	\caption{
	    Overall mechanism of ChoiceNet. It consists of $K$ mixtures and each mixture outputs triplet $(\pi_{k},\mu_{k},\Sigma_{k})$ via Algorithm \ref{alg:choice-net}. $\rho_{1}=1$ is reserved to model the target distribution. 
		}
	\label{fig:overall}
\end{figure}

We also propose the following Kullback-Leibler regularizer:
\begin{align}
    \label{KL-reg}
    \mathbb{KL}(\bar{\boldsymbol{\rho}} \| \boldsymbol{\pi} ) 
    = \sum_{k=1}^{K} \bar{\boldsymbol{\rho}}_{k}
    \log\frac{\bar{\boldsymbol{\rho}}_{k}}{\pi_{k}},
    \quad \bar{\boldsymbol{\rho}} 
    = \text{softmax}(\boldsymbol{\rho})
\end{align}
The above KL regularizer encourages the mixture components 
with the strong correlations to have high mixture probabilities.
Note that we use the notion of correlation to evaluate 
the goodness of each training data 
where the first mixture whose correlation
is always $1$ is reserved for modeling the target distribution.

%
%
\begin{figure*}[!t] 
	\centering 
	\subfigure[]{\includegraphics[width=.21\textwidth]
		{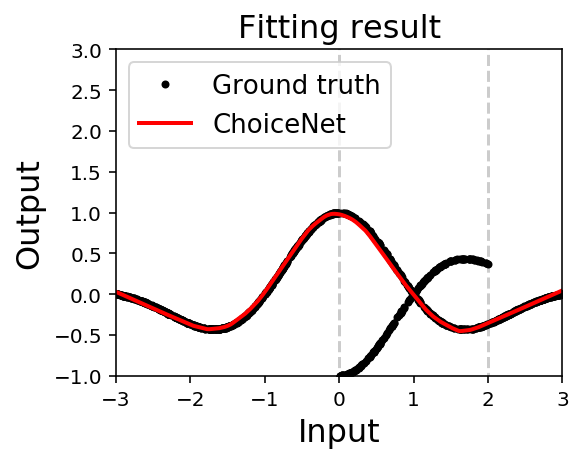}
		\label{fig:corr_check_a}}
	\subfigure[]{\includegraphics[width=.21\textwidth]
		{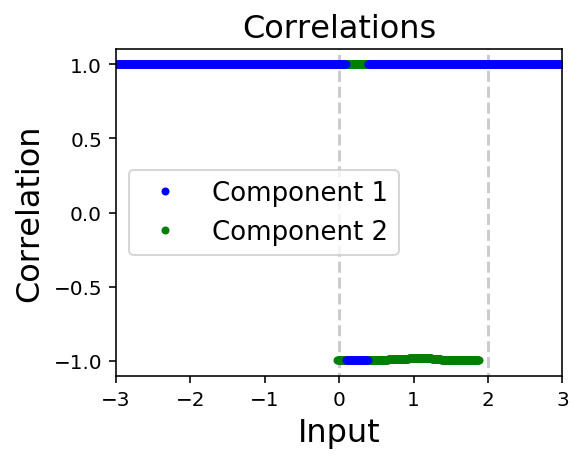}
		\label{fig:corr_check_b}}
	\subfigure[]{\includegraphics[width=.21\textwidth]
		{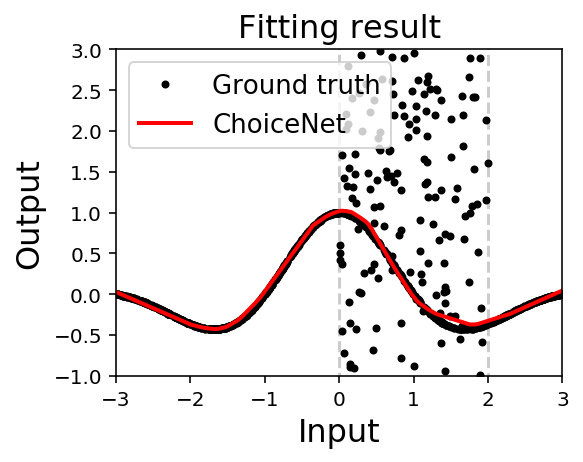}
		\label{fig:corr_check_c}}
	\subfigure[]{\includegraphics[width=.21\textwidth]
		{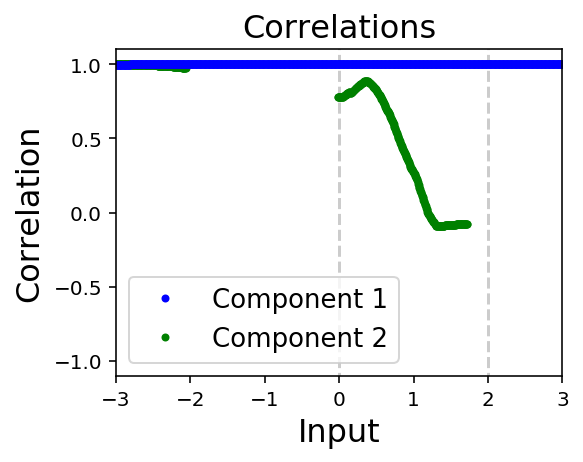}
		\label{fig:corr_check_d}}
	\caption{
	    Fitting results on datasets with 
	    (a) flipped function
	    and 
	    (c) uniform corruptions.
	    Resulting correlations of two components
	    with 
	    (b) flipped function
	    and 
	    (d) uniform corruptions.
		}
	\label{fig:corr_check}
\end{figure*}

%
%
\paragraph{Classification}

In the classification task, we suppose 
each $\mathbf{y}_{i}$ is a $D$-dimensional one-hot vector.
Unlike the regression task,
\eqref{loss-reg} is not appropriate for the classification task. 
We employ the following loss function:
\begin{align}
    \label{loss-cl-1}
    \mathcal{L}(\mathcal{D}) =
    -\frac{1}{N}\sum_{i=1}^{N} \sum_{k=1}^{K}\pi_{k}(\mathbf{x}_{i}) 
    	l(\mathbf{x}_i,\mathbf{y}_i)
\end{align}
where    
\begin{align*}
	l(\mathbf{x}_i,\mathbf{y}_i) & = 
    	\langle \text{softmax}(\hat{\mathbf{y}}_{k}(\mathbf{x}_{i})), \mathbf{y}_{i}\rangle \\ 
    	& \qquad -
	\lambda_\text{reg} \log\left(\sum_{d=1}^{D} \exp (\hat{y}_{k}^{(d)}(\mathbf{x}_{i})) 
	\right).
\end{align*}
Here $\langle \cdot,\cdot \rangle$ denotes inner product, 
$\hat{\mathbf{y}}_{k} = (\hat{y}_{k}^{(1)},\ldots,\hat{y}_{k}^{(D)})$, and
$\hat{y}_{k}^{(d)}(\mathbf{x}_{i}) = \mu_{k}^{(d)} + \sqrt{\Sigma_{k}^{(d)}} \varepsilon$
where $\varepsilon\sim\mathcal{N}(0,1)$.
Similar loss function was used in \cite{Kendall_17_unct} which also 
utilizes a Gaussian mixture model. 

%
%
\subsection{Illustrative Synthetic Examples}
\label{subsec:ill}

Here, we provide synthetic examples in both regression
and classification to illustrate 
how the proposed method
can be used to robustly learn the underlying target distribution 
given a noisy training dataset.

%
%
\paragraph{Regression Task}
\label{subsec:toy_reg}

We first focus on
how the proposed method can distinguish between
the target distribution and noise distributions
in a regression task and 
show empirical evidence that our method can
achieve this by estimating 
the target distribution and the correlation 
of noise distribution simultaneously. 
We train on two datasets
with the same target function but with 
different types of corruptions by
replacing $50\%$ of the output values 
whose input values are within $0$ to $2$:
one uniformly sampled from $-1$ to $3$
and the other from a flipped target function
as shown in Figure \ref{fig:corr_check_a}
and \ref{fig:corr_check_c}.
Throughout this experiment, we set $K = 2$ 
for better visualization.

As shown in Figure \ref{fig:corr_check},
ChoiceNet successfully estimates the target function
with partial corruption. 
To further understand how ChoiceNet works, 
we also plot the correlation of each mixture
at each input with different colors. 
When using the first dataset where we flip the outputs 
whose inputs are between 
$0$ and $2$, the correlations of the second 
mixture at the corrupt region becomes $-1$
(see Figure \ref{fig:corr_check_b}).
This is exactly what we wanted ChoiceNet to behave
in that having $-1$ correlation will simply flip the output.
In other words, the second mixture {\it takes care of}
the noisy (flipped) data by assigning $-1$ correlation
while the first mixture component reserved for
the target distribution is less affected by the noisy training data.

When using the second dataset,
the correlations of the second mixture at the corrupt region
are not $-1$ but decreases as the input increase from $0$
to $2$ (see Figure \ref{fig:corr_check_d}).
Intuitively speaking, this is because
the average deviations between the noisy output
and the target output increases as the input increases.
Since decreasing the correlation from $1$ to $0$ will
increase the output variance as shown in (\ref{eqn:sig_k}),
the correlation of the second mixture tends to decrease
as the input increase between $0$ to $2$. 
This clearly shows the capability of
ChoiceNet to distinguish the target distribution
from noisy distributions.

%
%
\begin{figure}[!t] 
	\centering 
	\subfigure[]{\includegraphics[width=.21\textwidth]
		{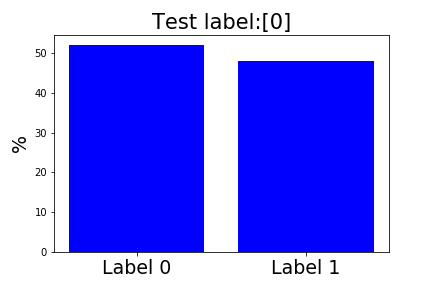}
		\label{fig:hist_check_a}}
	\subfigure[]{\includegraphics[width=.21\textwidth]
		{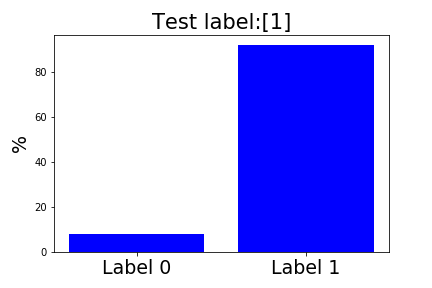}
		\label{fig:hist_check_b}}
	\caption{
	    The predictions results of the second mixture
	    of test inputs whose labels are
	    (a) $0$ and (b) $1$, respectively.
		}
	\label{fig:hist_check}
\end{figure}

%
%
\paragraph{Binary Classification Task}
\label{subsec:toy_cls}

We also provide an illustrative binary classification task
using label $0$ and $1$ from the MNIST dataset.
In particular, we replaced $40\%$ of the training data
with label $0$ to $1$. 
To implement ChoiceNet, 
we use two-layer convolutional neural networks with $64$ channels
and two mixtures.
We trained ChoiceNet for $10$ epochs where the final train and test accuracies
are $81.7\%$ and $98.1\%$, respectively, which indicates ChoiceNet
successfully infers clean data distribution.
As the first mixture is deserved for inferring the clean target distribution,
we expect the second mixture to take away corrupted labels. 
Figure \ref{fig:hist_check} shows two prediction results of the second mixture
when given test inputs with label $0$ and $1$. 
As $40\%$ of training labels whose labels are originally $0$ are replaced to $1$,
almost half of the second mixture predictions of test images whose labels are $0$ 
are $1$ indicating that it takes away corrupted labels while training. 
On the contrary, the second mixture predictions of test inputs whose original labels are $1$
are mostly focusing on label $1$.

%
%

\begin{figure*}[!t] 
	\centering 
	\subfigure[]{\includegraphics[width=.28\textwidth]
		{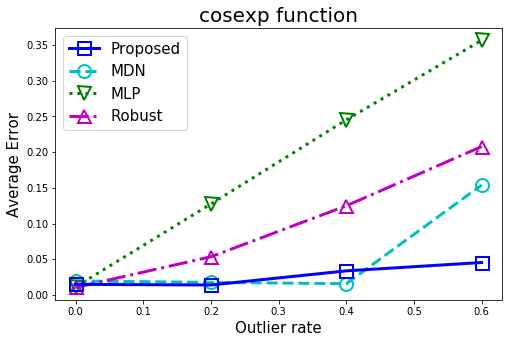}
		\label{fig:syn_ref_a}}
	\subfigure[]{\includegraphics[width=.28\textwidth]
		{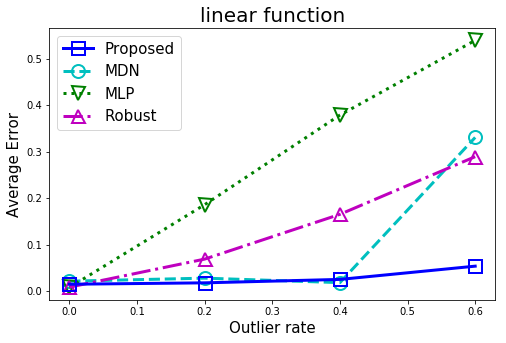}
		\label{fig:syn_ref_b}}
	\subfigure[]{\includegraphics[width=.28\textwidth]
		{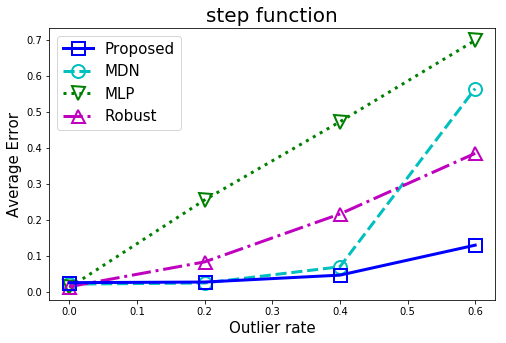}
		\label{fig:syn_ref_c}}
	\subfigure[]{\includegraphics[width=.28\textwidth]
		{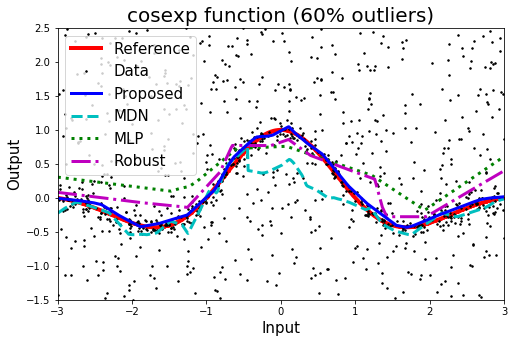}
		\label{fig:syn_ref_d}}
	\subfigure[]{\includegraphics[width=.28\textwidth]
		{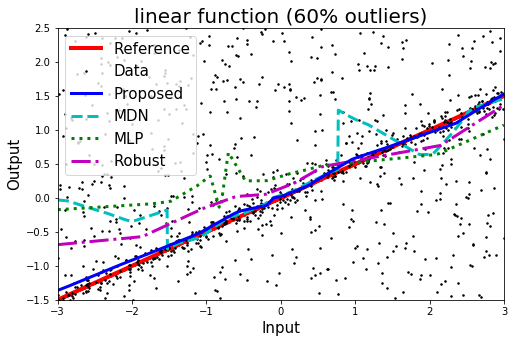}
		\label{fig:syn_ref_e}}
	\subfigure[]{\includegraphics[width=.28\textwidth]
		{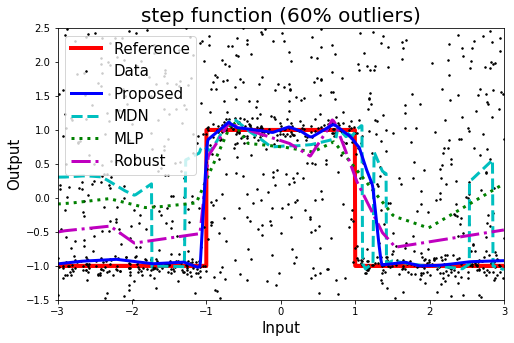}
		\label{fig:syn_ref_f}}
	\caption{
	    (a-c) Average fitting errors while varying the outlier rates
	    and 
	    (e-f) fitting results of the compared methods with $60\%$ outliers
	    using \emph{cosexp}, \emph{linear}, and
	    \emph{step} functions.
		}
	\label{fig:syn_ref}
\end{figure*}

%
%
\section{Experiments}
\label{sec:exp}

In this section, we validate the performance of ChoiceNet
on both regression problems in Section \ref{subsec:reg}
and classification problems in Section \ref{subsec:cls}
where we mainly focus on evaluating the robustness
of the proposed method.

%
%
\subsection{Regression Tasks} 
\label{subsec:reg}
Here, we show the regression performances using 
synthetic regression tasks and 
a Boston housing dataset with outliers.
More experimental results using synthetic data and
behavior cloning scenarios in MuJoCo environments
where the demonstrations are collected and mixed from 
both expert and adversarial policies
as well as detailed experimental settings 
can be found in the supplementary material.

%
%
\paragraph{Synthetic Data}

In order to evaluate the robustness of the proposed method,
we use three different $1$-D target functions.
Particularly, we use the following target functions:
\emph{cosexp}, \emph{linear}, and \emph{step} functions
as shown in Figure \ref{fig:syn_ref_d}-\ref{fig:syn_ref_f},
respectively, and collected $1,000$ points per each function
while replacing a certain portion of outputs to random values
uniformly sampled between $-1.5$ and $2.5$.
We compared our proposed method with
a mixture density network (MDN)
\cite{bishop1994mixture}
and fully connected layers with an $L_2$-loss (MLP)
and a robust loss (Robust) proposed in
\cite{Belagiannis_15_RobustReg}.
We use three-layers with $64$ units and ReLU activations,
and for both the proposed method and an MDN, 
five mixtures are used.

The average absolute fitting errors of three different functions
are shown in Figure \ref{fig:syn_ref_a}-\ref{fig:syn_ref_c},
respectively where we can see that the proposed method
outperforms or show comparable results with low outlier rates
and shows superior performances with a high outlier rate ($>50\%$).
Figure \ref{fig:syn_ref_d}-\ref{fig:syn_ref_f} show the fitting
results along with training data. 
While our proposed method is built on top of an MDN,
it is worthwhile noting the severe performance degradation
of an MDN with an extreme noise level ($60\%$).
It is mainly because an MDN fails to allocate 
high mixture probability on the mixture corresponds to 
the target distribution as there are no dependencies among
different mixtures.

%
%
\paragraph{Boston Housing Dataset}
\label{subsubsec:bh}
A Boston housing price dataset is used
to check the robustness of the proposed method.
We compare our method with standard feedforward networks using 
four different loss functions: 
standard $L2$-loss, $L1$-loss which is known to be robust to outliers, 
a robust loss (RL) function proposed in
\cite{Belagiannis_15_RobustReg}, and a leaky robust loss (LeakyRL) function
where the results are shown in Table \ref{tbl:boston}. 
Implementation details can be found in the supplementary material. 

%
%
\begin{table}[t!]  \center 
\small
\caption{The RMSEs of compared methods
	on the Boston Housing Dataset}
\label{tbl:boston}
\begin{tabular}{ l | p{1.0cm} p{0.5cm} p{0.5cm} p{0.5cm} p{0.9cm} p{0.9cm} p{0.9cm} } 
    \toprule
    Outliers & ChoiceNet & $L2$
                    & $L1$ & RL & LeakyRL &  MDN \\
	\midrule
	\textit{$0\%$}		
	& $3.29$  	& \bf{3.22}
	& $3.26$ 	& $4.28$  & $3.36$ & $3.46$
	\\ 
	\textit{$10\%$}		
	& \bf{3.99} 	 	& $5.97$
	& $5.72$ 	& $6.36$ 	& $5.71$ & $6.5$
	\\ 
	\textit{$20\%$}		
	& \bf{4.77} 	 	& $7.51$
	& $7.16$ 	& $8.08$	& $7.08$ & $8.62$
	\\ 
	\textit{$30\%$}		
	& \bf{5.94} 	 	& $9.04$
	& $8.65$ 	& $10.54$	& $8.67$ & $8.97$
	\\ 
	\textit{$40\%$}		
	& \bf{6.80} 		& $9.88$
	& $9.69$	& $10.94$	& $9.68$ & $10.44$
	\\
	\bottomrule
\end{tabular}
\end{table}

%
%
\begin{table}[t]
	\small
	\caption{
		Test accuracies on the CIFAR-10 dataset with by symmetric and asymmetric noises. 
	}
    \label{tab:cifar2}
    \centering
	\begin{tabular}{ l  l  l  l  }
	\toprule
      		       & Pair-$45\%$ & Sym-$50\%$ & Sym-$20\%$ \\ 
      		       \midrule
    Standard     & 49.5        & 48.87      & 76.25      \\ 
	+ ChoiceNet    & 70.3        & \bf{85.2}       & \bf{91.0}       \\ 
	\midrule
	MentorNet    & 58.14       & 71.10      & 80.76      \\ 
	Co-teaching  & \bf{72.62}       & 74.02      & 82.32      \\ 
	F-correction & 6.61        & 59.83      & 59.83      \\ 
	MDN          & 51.4        & 58.6      & 81.4      \\
	\bottomrule
\end{tabular}
\end{table}

%
%

\begin{table*}[t!] \center
\small
\begin{tabular}{c c | c c | c c | c c | c c | c c }
\toprule
  & & \multicolumn{2}{c|}{MNIST} & \multicolumn{2}{c|}{FMNIST} & \multicolumn{2}{c|}{SVHN} & \multicolumn{2}{c|}{CIFAR10} & \multicolumn{2}{c}{CIFAR100} \\
 Data & Model & 20\% & 50\% & 20\% & 50\% & 20\% & 50\% & 20\% & 50\% & 20\% & 50\%  
\tabularnewline
\midrule
 \multirow{2}{*}{$D_{\text{test}}$} & WideResNet & 82.19 & 56.63 & 82.32 & 56.55 & 82.21 & 55.94 & 81.93 & 54.97 & 80.19 & 50.48 \\
 & + ChoiceNet & 99.66 & 99.03 & 97.46 & 95.36 & 94.40 & 78.32 & 96.58 & 90.28 & 85.81 & 68.89 \\
 \midrule
 \multirow{2}{*}{$D_{\text{train}}$}  & WideResNet & 99.29 & 92.49 & 98.62 & 92.42 & 99.34 & 95.91 & 99.97 & 99.82 & 99.96 & 99.91 \\
 & + ChoiceNet & 81.99 & 55.38 & 80.63 & 54.45 & 78.99 & 48.69 & 81.88 & 56.83 & 26.03 & 18.14 \\
\midrule
 \multicolumn{2}{c|}{Expected True Ratio} & 82\% & 55\% & 82\% & 55\% & 82\% & 55\% & 82 \% & 55\% & 80.2\% & 50.5 \% 
\tabularnewline
\bottomrule
\end{tabular}
\caption{The comparison between naive WideResNet and ChoiceNet on multile benchmark datasets.}  
\label{table_cls_valid}
\end{table*}

\subsection{Classification Tasks}
\label{subsec:cls}

%
%



Here, we conduct comprehensive classification experiments
to investigate how ChoiceNet performs on image classification tasks
with noisy labels.
More experimental results on MNIST, CIFAR-10, and
a natural language processing task
and ablation studies of hyper-parameters
can also be found in the supplementary material. 


\paragraph{CIFAR-10 Comparison}

We first evaluate the performance of our method
and compare it with MentorNet \cite{Jiang_17_mentornet},
Co-teaching \cite{Han_18} and F-correction \cite{Patrini_17_LossCorrection} on noisy CIFAR-10 datasets.
We follow three different noise settings from \cite{Han_18}: PairFlip-$45\%$, Symmetry-$50\%$, and Symmetry-$20\%$. On `Symmetry` settings, noisy labels can be assigned to all classes, while, on `PairFlip` settings, all noisy labels from the same true label are assigned to a single noisy label. A model is trained on a noisy dataset and evaluated on the clean test set.
For fair comparisons, we keep other configurations such as the network topology to be the same as \cite{Han_18}.

Table \ref{tab:cifar2} shows the test accuracy of compared methods under different noise settings. Our proposed method outperforms all compared methods on the symmetric noise settings with a large margin over 10\%p. On asymmetric noise settings (Pair-$45\%$), our method shows the second best performance, and this reveals the weakness of the proposed method. As Pair-$45\%$ assigns $45\%$ of each label to its next label, 
The \emph{Cholesky Block} fails to infer the dominant label distributions
correctly.
However, we would like to note that Co-teaching \cite{Han_18}
is complementary to our method
where one can combine these two methods by using two ChoiceNets 
and update each network using Co-teaching.

\paragraph{Generalization to Other Datasets}

We conduct additional experiments to investigate 
how our method works on other image datasets.
We adopt the structure of WideResNet \cite{Zagoruyko2016WRN} 
and design a baseline network with a depth of $22$ 
and a widen factor of $4$.
We also construct ChoiceNet by replacing the last two layers 
(`average pooling` and `linear`) with the \emph{Cholesky Block}. 
We train the networks on CIFAR-10, CIFAR-100, FMNIST, MNIST and SVHN datasets
with noisy labels. 
We generate noisy datasets with symmetric noise setting and 
vary corruption ratios from $20\%$ to $50\%$. 
We would like to emphasize that we use the same hyper-parameters, 
which are tuned for the CIFAR-10 experiments, 
for all the datasets except for CIFAR-100\footnote{
On CIFAR-100 dataset, we found the network underfits with our default
hyper-parameters. 
Therefore, we enlarged $\lambda_{reg}$ to $1e-3$ for CIFAR-100 experiments.}.

Table \ref{table_cls_valid} shows test accuracies of our method
and the baseline on various image datasets with noisy labels.
ChoiceNet consistently outperforms the baseline in most of the configurations
except for $80\%$-corrupted SVHN dataset. 
Moreover, we expect that performance gains can increase when 
dataset-specific hyperparameter tuning is applied. 
The results suggest that the proposed ChoiceNet 
can easily be applied to other noisy datasets 
and show a performance improvement without large efforts.

We would like to emphasize that the training accuracy
of the proposed method in Table \ref{table_cls_valid}
is close to the expected true ratio\footnote{
Since a noisy label can be assigned to any labels,
we can expect $82\%$, $55\%$ true labels on noisy datasets 
with a corruption probability of $20\%$, $50\%$, respectively.}. 
This implies that our proposed \emph{Cholesky Block} can separate true labels 
and false labels from a noisy dataset.
We also note that training ChoiceNet on CIFAR-100 datasets
requires a modification in a loss weight.
We hypothesize that the hyperparameters of our proposed method 
are sensitive to a number of target classes in a dataset.

\section{Conclusion}

In this paper, we have presented ChoiceNet
that can robustly learn a target distribution
given noisy training data.
The keystone of ChoiceNet is the 
mixture of correlated densities network
block which can simultaneously estimate 
the underlying target distribution and 
the quality of each data where the quality
is defined by the correlation between
the target and generating distributions.
We have demonstrated that the proposed method can robustly infer
the target distribution
on corrupt training data in both regression and classification tasks.
However, we have seen that in the case of extreme asymmetric noises,
the proposed method showed suboptimal performances.
We believe that it could resolved by combining our method
with other robust learning methods
where we have demonstrated that ChoiceNet can 
effectively be combined with mix-up \cite{Zhang_18_mixup}. 
Furthermore, one can use ChoiceNet for active learning
by evaluating the quality of each training data using 
through the lens of correlations.

\medskip

\small

\newpage

\bibliographystyle{ieeepes}
\bibliography{bib}

\begin{thebibliography}{10}

\bibitem{Bi_14_Crowdsource}
Wei Bi, Liwei Wang, James~T Kwok, and Zhuowen Tu,
\newblock ``Learning to predict from crowdsourced data.'',
\newblock in {\em UAI}, 2014, pp. 82--91.

\bibitem{Zhang_17_rethinking}
Chiyuan Zhang, Samy Bengio, Moritz Hardt, Benjamin Recht, and Oriol Vinyals,
\newblock ``Understanding deep learning requires rethinking generalization'',
\newblock in {\em Proc. of International Conference on Learning
  Representations}, 2016.

\bibitem{Hampel_11_Robust}
Frank~R Hampel, Elvezio~M Ronchetti, Peter~J Rousseeuw, and Werner~A Stahel,
\newblock {\em Robust statistics: the approach based on influence functions},
  vol. 196,
\newblock John Wiley \& Sons, 2011.

\bibitem{Jindal_16_noisyLabel}
Ishan Jindal, Matthew Nokleby, and Xuewen Chen,
\newblock ``Learning deep networks from noisy labels with dropout
  regularization'',
\newblock in {\em Proc. of IEEE International Conference onData Mining}. IEEE,
  2016, pp. 967--972.

\bibitem{Blundell_15}
Charles Blundell, Julien Cornebise, Koray Kavukcuoglu, and Daan Wierstra,
\newblock ``Weight uncertainty in neural networks'',
\newblock in {\em International Conference on Machine Learning (ICML)}, 2015.

\bibitem{Fawzi_17_robustness}
Alhussein Fawzi, Seyed Mohsen~Moosavi Dezfooli, and Pascal Frossard,
\newblock ``A geometric perspective on the robustness of deep networks'',
\newblock {\em IEEE Signal Processing Magazine}, 2017.

\bibitem{Paden_16}
Brian Paden, Michal {\v{C}}{\'a}p, Sze~Zheng Yong, Dmitry Yershov, and Emilio
  Frazzoli,
\newblock ``A survey of motion planning and control techniques for self-driving
  urban vehicles'',
\newblock {\em IEEE Transactions on Intelligent Vehicles}, vol. 1, no. 1, pp.
  33--55, 2016.

\bibitem{Gulshan_16_MedicalDeepLearning}
Varun Gulshan, Lily Peng, Marc Coram, Martin~C Stumpe, Derek Wu, Arunachalam
  Narayanaswamy, Subhashini Venugopalan, Kasumi Widner, Tom Madams, Jorge
  Cuadros, et~al.,
\newblock ``Development and validation of a deep learning algorithm for
  detection of diabetic retinopathy in retinal fundus photographs'',
\newblock {\em Journal of the American Medical Association}, vol. 316, no. 22,
  pp. 2402--2410, 2016.

\bibitem{Tesla_16}
AP and REUTERS,
\newblock ``Tesla working on 'improvements' to its autopilot radar changes
  after model s owner became the first self-driving fatality.'', June 2016.

\bibitem{Jiang_17_mentornet}
Lu~Jiang, Zhengyuan Zhou, Thomas Leung, Li-Jia Li, and Li~Fei-Fei,
\newblock ``Mentornet: Regularizing very deep neural networks on corrupted
  labels'',
\newblock {\em arXiv preprint arXiv:1712.05055}, 2017.

\bibitem{Ren_18}
Mengye Ren, Wenyuan Zeng, Bin Yang, and Raquel Urtasun,
\newblock ``Learning to reweight examples for robust deep learning'',
\newblock in {\em Proc. of International Conference on Machine Learning}, 2018.

\bibitem{Han_18}
Bo~Han, Quanming Yao, Xingrui Yu, Gang Niu, Miao Xu, Weihua Hu, Ivor Tsang, and
  Masashi Sugiyama,
\newblock ``Co-teaching: robust training deep neural networks with extremely
  noisy labels'',
\newblock in {\em Proc. of the Advances in Neural Information Processing
  Systems}, 2018.

\bibitem{Malach_17_Decoupling}
Eran Malach and Shai Shalev-Shwartz,
\newblock ``Decoupling" when to update" from" how to update"'',
\newblock in {\em Advances in Neural Information Processing Systems}, 2017, pp.
  961--971.

\bibitem{Patrini_17_LossCorrection}
Giorgio Patrini, Alessandro Rozza, Aditya~Krishna Menon, Richard Nock, and
  Lizhen Qu,
\newblock ``Making deep neural networks robust to label noise: a loss
  correction approach'',
\newblock in {\em Proc. of the Conference on Computer Vision and Pattern
  Recognition}, 2017, vol. 1050, p.~22.

\bibitem{Goldberger_17_NoiseAdaptation}
Jacob Goldberger and Ehud Ben-Reuven,
\newblock ``Training deep neural-networks using a noise adaptation layer'',
\newblock in {\em Proc. of International Conference on Learning
  Representations}, 2017.

\bibitem{Sukhbaatar_14}
Sainbayar Sukhbaatar, Joan Bruna, Manohar Paluri, Lubomir Bourdev, and Rob
  Fergus,
\newblock ``Training convolutional networks with noisy labels'',
\newblock {\em arXiv preprint arXiv:1406.2080}, 2014.

\bibitem{Bekker_16_unreliableLabel}
Alan~Joseph Bekker and Jacob Goldberger,
\newblock ``Training deep neural-networks based on unreliable labels'',
\newblock in {\em Proc. of IEEE International Conference on Acoustics, Speech
  and Signal Processing}. IEEE, 2016, pp. 2682--2686.

\bibitem{Hendrycks_18_GLC}
Dan Hendrycks, Mantas Mazeika, Duncan Wilson, and Kevin Gimpel,
\newblock ``Using trusted data to train deep networks on labels corrupted by
  severe noise'',
\newblock {\em arXiv preprint arXiv:1802.05300}, 2018.

\bibitem{Veit_17_noisy}
Andreas Veit, Neil Alldrin, Gal Chechik, Ivan Krasin, Abhinav Gupta, and Serge
  Belongie,
\newblock ``Learning from noisy large-scale datasets with minimal
  supervision'',
\newblock in {\em Conference on Computer Vision and Pattern Recognition}, 2017.

\bibitem{Natarajan_13_NoisyLabel}
Nagarajan Natarajan, Inderjit~S Dhillon, Pradeep~K Ravikumar, and Ambuj Tewari,
\newblock ``Learning with noisy labels'',
\newblock in {\em Proc. of the Advances in Neural Information Processing
  Systems}, 2013, pp. 1196--1204.

\bibitem{Belagiannis_15_RobustReg}
Vasileios Belagiannis, Christian Rupprecht, Gustavo Carneiro, and Nassir Navab,
\newblock ``Robust optimization for deep regression'',
\newblock in {\em Proc. of the IEEE International Conference on Computer
  Vision}, 2015, pp. 2830--2838.

\bibitem{Reed_14_bootstrap}
Scott Reed, Honglak Lee, Dragomir Anguelov, Christian Szegedy, Dumitru Erhan,
  and Andrew Rabinovich,
\newblock ``Training deep neural networks on noisy labels with bootstrapping'',
\newblock {\em arXiv preprint arXiv:1412.6596}, 2014.

\bibitem{Lee_13_pseudo}
Dong-Hyun Lee,
\newblock ``Pseudo-label: The simple and efficient semi-supervised learning
  method for deep neural networks'',
\newblock in {\em Workshop on Challenges in Representation Learning, ICML},
  2013, vol.~3, p.~2.

\bibitem{Goodfellow_16_DLbook}
Ian Goodfellow, Yoshua Bengio, Aaron Courville, and Yoshua Bengio,
\newblock {\em Deep learning}, vol.~1,
\newblock MIT press Cambridge, 2016.

\bibitem{Xie_16_disturblabel}
Lingxi Xie, Jingdong Wang, Zhen Wei, Meng Wang, and Qi~Tian,
\newblock ``Disturblabel: Regularizing cnn on the loss layer'',
\newblock in {\em Proc. of the IEEE Conference on Computer Vision and Pattern
  Recognition}, 2016, pp. 4753--4762.

\bibitem{Tokozume_18_btwClassEx}
Yuji Tokozume, Yoshitaka Ushiku, and Tatsuya Harada,
\newblock ``Proc. of international conference on learning representations'',
\newblock 2018.

\bibitem{Zhang_18_mixup}
Hongyi Zhang, Moustapha Cisse, Yann~N Dauphin, and David Lopez-Paz,
\newblock ``mixup: Beyond empirical risk minimization'',
\newblock in {\em Proc. of International Conference on Learning
  Representations}, 2017.

\bibitem{Miyato18}
Takeru Miyato, Shin-ichi Maeda, Shin Ishii, and Masanori Koyama,
\newblock ``Virtual adversarial training: a regularization method for
  supervised and semi-supervised learning'',
\newblock {\em IEEE transactions on pattern analysis and machine intelligence},
  2018.

\bibitem{Tarvainen_17}
Antti Tarvainen and Harri Valpola,
\newblock ``Mean teachers are better role models: Weight-averaged consistency
  targets improve semi-supervised deep learning results'',
\newblock in {\em Advances in neural information processing systems}, 2017, pp.
  1195--1204.

\bibitem{Laine_17}
Samuli Laine and Timo Aila,
\newblock ``Temporal ensembling for semi-supervised learning'',
\newblock in {\em Proc. of International Conference on Learning
  Representations}, 2017.

\bibitem{Liu_17_SelfCorrect}
Xin Liu, Shaoxin Li, Meina Kan, Shiguang Shan, and Xilin Chen,
\newblock ``Self-error-correcting convolutional neural network for learning
  with noisy labels'',
\newblock in {\em Proc. of IEEE International Conference on Automatic Face
  \&amp; Gesture Recognition}. IEEE, 2017, pp. 111--117.

\bibitem{Ghosh_17_robustLoss}
Aritra Ghosh, Himanshu Kumar, and PS~Sastry,
\newblock ``Robust loss functions under label noise for deep neural
  networks.'',
\newblock in {\em Proc. of the AAAI Conference on Artificial Intelligence},
  2017, pp. 1919--1925.

\bibitem{Rasmussen_06}
Carl~Edward Rasmussen,
\newblock ``Gaussian processes for machine learning'',
\newblock 2006.

\bibitem{Bonilla_08_MTGPP}
Edwin~V Bonilla, Kian~M Chai, and Christopher Williams,
\newblock ``Multi-task gaussian process prediction'',
\newblock in {\em Proc. of the Advances in Neural Information Processing
  Systems}, 2008, pp. 153--160.

\bibitem{SJChoi_16}
Sungjoon Choi, Kyungjae Lee, and Songhwai Oh,
\newblock ``Robust learning from demonstration using leveraged {G}aussian
  processes and sparse constrained opimization'',
\newblock in {\em Proc. of the IEEE International Conference on Robotics and
  Automation (ICRA)}. May 2016, IEEE.

\bibitem{Huber_11}
Peter~J Huber,
\newblock {\em Robust statistics},
\newblock Springer, 2011.

\bibitem{kingma2017variational}
Diederik~P Kingma,
\newblock ``Variational inference \& deep learning: A new synthesis'',
\newblock {\em University of Amsterdam}, 2017.

\bibitem{bishop1994mixture}
Christopher~M Bishop,
\newblock ``Mixture density networks'',
\newblock 1994.

\bibitem{He_16}
Kaiming He, Xiangyu Zhang, Shaoqing Ren, and Jian Sun,
\newblock ``Deep residual learning for image recognition'',
\newblock in {\em Proc. of the IEEE conference on Computer Vision and Pattern
  Recognition}, 2016, pp. 770--778.

\bibitem{Levy_14}
Omer Levy and Yoav Goldberg,
\newblock ``Neural word embedding as implicit matrix factorization'',
\newblock in {\em Advances in neural information processing systems}, 2014, pp.
  2177--2185.

\bibitem{kendall2017uncertainties}
Alex Kendall and Yarin Gal,
\newblock ``What uncertainties do we need in bayesian deep learning for
  computer vision?'',
\newblock in {\em Advances in Neural Information Processing Systems}, 2017, pp.
  5580--5590.

\bibitem{choi2017uncertainty}
Sungjoon Choi, Kyungjae Lee, Sungbin Lim, and Songhwai Oh,
\newblock ``Uncertainty-aware learning from demonstration using mixture density
  networks with sampling-free variance modeling'',
\newblock {\em arXiv preprint arXiv:1709.02249}, 2017.

\bibitem{christopher2016pattern}
M~Bishop Christopher,
\newblock {\em PATTERN RECOGNITION AND MACHINE LEARNING.},
\newblock Springer-Verlag New York, 2016.

\bibitem{Kendall_17_unct}
Alex Kendall and Yarin Gal,
\newblock ``What uncertainties do we need in {B}ayesian deep learning for
  computer vision?'',
\newblock in {\em Proc. of the Advances in Neural Information Processing
  Systems}, 2017.

\bibitem{Zagoruyko2016WRN}
Sergey Zagoruyko and Nikos Komodakis,
\newblock ``Wide residual networks'',
\newblock in {\em BMVC}, 2016.

\bibitem{Rasmussen_02_imgp}
Carl~E Rasmussen and Zoubin Ghahramani,
\newblock ``Infinite mixtures of gaussian process experts'',
\newblock in {\em Advances in Neural Information Processing Systems}, 2002, pp.
  881--888.

\bibitem{Dai_14}
Bo~Dai, Bo~Xie, Niao He, Yingyu Liang, Anant Raj, Maria-Florina~F Balcan, and
  Le~Song,
\newblock ``Scalable kernel methods via doubly stochastic gradients'',
\newblock in {\em Proc. of the Advances in Neural Information Processing
  Systems}, 2014, pp. 3041--3049.

\bibitem{Schulman_17proximal}
John Schulman, Filip Wolski, Prafulla Dhariwal, Alec Radford, and Oleg Klimov,
\newblock ``Proximal policy optimization algorithms'',
\newblock {\em arXiv preprint arXiv:1707.06347}, 2017.

\bibitem{he2016deep}
Kaiming He, Xiangyu Zhang, Shaoqing Ren, and Jian Sun,
\newblock ``Deep residual learning for image recognition'',
\newblock in {\em Proceedings of the IEEE conference on computer vision and
  pattern recognition}, 2016, pp. 770--778.

\bibitem{Bengio_03}
Yoshua Bengio, R{\'e}jean Ducharme, Pascal Vincent, and Christian Jauvin,
\newblock ``A neural probabilistic language model'',
\newblock {\em Journal of machine learning research}, vol. 3, no. Feb, pp.
  1137--1155, 2003.

\end{thebibliography}

\newpage
\onecolumn

\appendix

%
%
%
%
\section{Proof of Theorem in Section 3}

In this appendix, we introduce fundamental theorems which lead to
Cholesky transform for given random variables $(W, Z)$. 
We apply this transform to random matrices
$\mathbf{W}$ and $\mathbf{Z}$ which carry out weight matrices
for prediction and a supplementary role, respectively.
We also elaborate the details of 
conducted experiments with additional illustrative
figures and results.
Particularly, we show additional classification 
experiments with the MNST dataset on different 
noise configurations. 

\begin{lem}
\label{chd}
Let $W$ and $Z$ be uncorrelated random variables such that 
\begin{equation}
\begin{cases}
\mathbb{E}W=\mu_{W}, & \mathbb{V}\left(W\right)=\sigma_{W}^{2}\\
\mathbb{E}Z=0, & \mathbb{V}\left(Z\right)=\sigma_{Z}^{2}
\end{cases}\label{eq:assum_chol}
\end{equation}
For a given $-1 \leq\rho \leq 1$, set
\begin{align}
\label{tilde_Z}
\tilde{Z}=\rho\frac{\sigma_{Z}}{\sigma_{W}}(W-\mu_{W})+\sqrt{1-\rho^{2}}Z
\end{align}
Then $\mathbb{E}\tilde{Z} =0$, $\mathbb{V}(\tilde{Z}) = \sigma_{Z}^{2}$, and  $\text{\rm{Corr}} (W,\tilde{Z})=\rho$.
\end{lem}

\begin{proof}
Since $W$ and $Z$ are uncorrelated, we have
\begin{align}
\mathbb{E}\left[(W-\mu_{W})Z\right] = \mathbb{E}(W-\mu_{W})\mathbb{E}Z = 0 \label{eq:uncorr}
\end{align}
By \eqref{eq:assum_chol}, we directly obtain
\[
\mathbb{E}\tilde{Z}=\rho\frac{\sigma_{Z}}{\sigma_{W}}\left(\mathbb{E}W-\mu_{W}\right)+\mathbb{E}Z=0
\]
Also, by \eqref{eq:assum_chol} and \eqref{eq:uncorr},
\begin{align*}
\mathbb{V}\left(\tilde{Z}\right)=\mathbb{E}|\tilde{Z}|^{2} & =\rho^{2}\left(\frac{\sigma_{Z}}{\sigma_{W}}\right)^{2}\mathbb{V}(W)+\mathbb{V}(Z)+2\rho\frac{\sigma_{Z}}{\sigma_{W}}\underbrace{\mathbb{E}\left[(W-\mu_{W})Z\right]}_{=0}\\
 & =\rho^{2}\frac{\sigma_{Z}^{2}}{\sigma_{W}^{2}}\sigma_{W}^{2}+(1-\rho^{2})\sigma_{Z}^{2}=\sigma_{Z}^{2}
\end{align*}
Similarly,
\begin{align*}
\text{Cov}(W,\tilde{Z}) & =\mathbb{E}\left[(W-\mu_{W})\tilde{Z}\right]\\
 & =\mathbb{E}\left[(W-\mu_{W})\rho\frac{\sigma_{Z}}{\sigma_{W}}(W-\mu_{W})\right]+\underbrace{\mathbb{E}\left[(W-\mu_{W})Z\right]}_{=0}\\
 & =\rho\frac{\sigma_{Z}}{\sigma_{W}}\mathbb{V}(W)=\rho\sigma_{Z}\sigma_{W}
\end{align*}
Therefore
\[
\text{Corr}(W,\tilde{Z})=\frac{\text{Cov}(W,\tilde{Z})}{\sqrt{\mathbb{V}(W)}\sqrt{\mathbb{V}(\tilde{Z})}}=\frac{\rho\sigma_{W}\sigma_{Z}}{\sigma_{W}\sigma_{Z}}=\rho
\]
The lemma is proved.
\end{proof}

\begin{lem}
\label{transform}
Assume the same condition in Lemma \ref{chd} and define $\tilde{Z}$ as \eqref{tilde_Z}. For given functions $\varphi : \mathbb{R}\to\mathbb{R}$ and $\psi : \mathbb{R} \to (0,\infty)$, set $\tilde{W}:=\varphi(\rho)+\psi(\rho)\tilde{Z}$. Then
\[
\mathbb{E}\tilde{W} = \varphi(\rho),\quad \mathbb{V}(\tilde{W}) = |\psi(\rho)|^{2}\sigma_{Z}^{2},\quad\text{\rm{Corr}}(W,\tilde{W})=\rho
\]
\end{lem}

\begin{proof}
Note that
\[
\mu_{\tilde{W}}=\varphi(\rho)+\psi(\rho)\mu_{\tilde{Z}} = \varphi(\rho)
\]
\[
\sigma_{\tilde{W}}^{2}=\left|\psi(\rho)\right|^{2}\mathbb{E}\left(\tilde{Z}-\mu_{\tilde{Z}}\right)^{2}=\psi^{2}(\rho)\sigma_{Z}^{2}
\]
Therefore, by Lemma \ref{chd}
\begin{align*}
\mathbb{E}\left[(W-\mu_{W})(\tilde{W}-\mu_{\tilde{W}})\right] & =\psi(\rho)\mathbb{E}\left[(W-\mu_{W})(\tilde{Z}-\mu_{\tilde{Z}})\right]\\
 & =\rho\psi(\rho)\sigma_{W}\sigma_{Z}
\end{align*}
Hence
\[
\text{\rm{Corr}}(W, \tilde{W})=\frac{\mathbb{E}\left[(W-\mu_{W})(\tilde{W}-\mu_{\tilde{W}})\right]}{\sigma_{W}\sigma_{\tilde{W}}}=\frac{\rho\psi(\rho)\sigma_{W}\sigma_{Z}}{\psi(\rho)\sigma_{W}\sigma_{Z}}=\rho
\]
The lemma is proved.
\end{proof}

Now we prove the aforementioned theorem in Section 3.

\begin{theorem*}

Let $\boldsymbol{\rho}=(\rho_{1},\ldots,\rho_{K})\in\mathbb{R}^{K}$.
For $p\in\{1,2\}$, random matrices $\mathbf{W}^{(p)}\in\mathbb{R}^{K\times Q}$ are given
such that for every $k\in\{1,\ldots,K\}$,
\begin{equation}
\text{\rm{Cov}}\left(W_{ki}^{(p)},W_{kj}^{(p)}\right)=\sigma_{p}^{2}\delta_{ij}, \quad \text{\rm{Cov}}\left(W_{ki}^{(1)},W_{kj}^{(2)}\right)=\rho_{k}\sigma_{1}\sigma_{2}\delta_{ij}
\label{eq:assumption1}
\end{equation}


Given $\mathbf{h} = (h_{1},\ldots,h_{Q}) \in\mathbb{R}^{Q}$, set $\mathbf{y}^{(p)}=\mathbf{W}^{(p)}\mathbf{h}$
for each $p\in\{1,2\}$. 
Then an elementwise
correlation between $\mathbf{y}^{(1)}$ and $\mathbf{y}^{(2)}$ equals
$\boldsymbol{\rho}$ i.e. 
\[
\text{\rm{Corr}}\left(y_{k}^{(1)},y_{k}^{(2)}\right)=\rho_{k},\quad \forall k\in \{1, \ldots, K\}
\]

\end{theorem*}

\begin{proof}
First we prove that for $p\in\{1,2\}$ and $k\in\{1,\ldots,K\}$
\begin{equation}
\mathbb{V}\left(y_{k}^{(p)}\right)=\sigma_{p}^{2}\left\Vert \mathbf{h}\right\Vert ^{2}\label{eq:var}
\end{equation}
Note that
\begin{align*}
\mathbb{V}\left(y_{k}^{(p)}\right) & =\mathbb{E}\left[\left(\sum_{i=1}^{Q}W_{ki}^{(p)}h_{i}-\mathbb{E}\left[\sum_{i=1}^{Q}W_{ki}^{(p)}h_{i}\right]\right)^{2}\right]\\
 & =\mathbb{E}\left[\left(\sum_{i=1}^{Q}\left(W_{ki}^{(p)}-\mathbb{E}W_{ki}^{(p)}\right)h_{i}\right)^{2}\right]\\
 & =\mathbb{E}\left[\sum_{i,j}^{Q}\left(W_{ki}^{(p)}-\mathbb{E}W_{ki}^{(p)}\right)\left(W_{kj}^{(p)}-\mathbb{E}W_{kj}^{(p)}\right)h_{i}h_{j}\right]\\
 & =\sum_{i,j}^{Q}\text{Cov}(W_{ki}^{(p)},W_{kj}^{(p)})h_{i}h_{j}
\end{align*}
By \eqref{eq:assumption1},
\[
\mathbb{V}\left(y_{k}^{(p)}\right)=\sum_{i,j}^{Q}\text{Cov}(W_{ki}^{(p)},W_{kj}^{(p)})h_{i}h_{j}=\sum_{i,j}^{Q}\sigma_{p}^{2}h_{i}h_{j}\delta_{ij}=\sum_{i=1}^{Q}\sigma_{p}^{2}h_{i}^{2}=\sigma_{p}^{2}\|\mathbf{h}\|^{2}
\]
so \eqref{eq:var} is proved. Next we prove 
\begin{equation}
\text{Cov}(y_{k}^{(1)},y_{k}^{(2)})=\rho_{k}\sigma_{1}\sigma_{2}\left\Vert \mathbf{h}\right\Vert ^{2}\label{eq:covar}
\end{equation}
Observe that
\begin{align*}
\text{Cov}(y_{k}^{(1)},y_{k}^{(2)}) & =\mathbb{E}\left[\left(y_{k}^{(1)}-\mathbb{E}y_{k}^{(1)}\right)\left(y_{k}^{(2)}-\mathbb{E}y_{k}^{(2)}\right)\right]\\
 & =\mathbb{E}\left[\left(\sum_{i=1}^{Q}W_{ki}^{(1)}h_{i}-\mathbb{E}\left[\sum_{i=1}^{Q}W_{ki}^{(1)}h_{i}\right]\right)\left(\sum_{j=1}^{Q}W_{kj}^{(2)}h_{j}-\mathbb{E}\left[\sum_{j=1}^{Q}W_{kj}^{(2)}h_{j}\right]\right)\right]\\
 & =\mathbb{E}\left[\sum_{i,j}^{Q}\left(W_{ki}^{(1)}-\mathbb{E}W_{ki}^{(1)}\right)\left(W_{kj}^{(2)}-\mathbb{E}W_{kj}^{(2)}\right)h_{i}h_{j}\right]\\
 & =\sum_{i,j}^{Q}\text{Cov}(W_{ki}^{(1)},W_{kj}^{(2)})h_{i}h_{j}
\end{align*}
Similarly, 
\[
\text{Cov}(y_{k}^{(1)},y_{k}^{(2)})=\sum_{i,j}^{Q}\text{Cov}(W_{ki}^{(1)},W_{kj}^{(2)})h_{i}h_{j}=\sum_{i,j}^{Q}\rho_{k}\sigma_{1}\sigma_{2}h_{i}h_{j}\delta_{ij}=\rho_{k}\sigma_{1}\sigma_{2}\left\Vert \mathbf{h}\right\Vert ^{2}
\]
Hence \eqref{eq:covar} is proved. Therefore by \eqref{eq:var} and \eqref{eq:covar}
\[
\text{\rm{Corr}}(y_{k}^{(1)},y_{k}^{(2)})=\frac{\text{Cov}(y_{k}^{(1)},y_{k}^{(2)})}{\sqrt{\mathbb{V}(y_{k}^{(1)})}\sqrt{\mathbb{V}(y_{k}^{(2)})}}=\frac{\rho_{k}\sigma_{1}\sigma_{2}\left\Vert \mathbf{h}\right\Vert ^{2}}{\sqrt{\sigma_{1}^{2}\left\Vert \mathbf{h}\right\Vert ^{2}}\sqrt{\sigma_{2}^{2}\left\Vert \mathbf{h}\right\Vert ^{2}}}=\rho_{k}
\]
The theorem is proved.
\end{proof}

%
%
\begin{rem*}
Recall the definition of Cholesky transform: for $-1 < \rho < 1$
\begin{align}
\label{ch-transform}
\mathscr{T}_{(\rho, \mu_{W},\sigma_{W},\sigma_{Z})}(w,z):=\rho\mu_{W}+\sqrt{1-\rho^{2}}\left(\rho\frac{\sigma_{Z}}{\sigma_{W}}(w-\mu_{W})+ \sqrt{1-\rho^{2}} z\right)
\end{align}
Note that we do not assume $W$ and $Z$ should follow typical distributions. Hence every above theorems hold for general class of random variables. Additionally, by Theorem \ref{transform} and \eqref{ch-transform}, $\tilde{W}$ has the following $\rho$-dependent behaviors;
\[
\mathbb{E}\tilde{W}\to
\begin{cases}
\mu_{W} &: \rho \to 1 \\
0 &: \rho \to 0 \\
-\mu_{W} &: \rho \to -1
\end{cases}
,\quad \mathbb{V}(\tilde{W})\to
\begin{cases}
0 &: \rho \to \pm 1 \\
\sigma_{Z}^{2} &: \rho \to 0
\end{cases}
\]
Thus strongly correlated weights $\tilde{W}$ i.e. $\rho\approx 1$, provide prediction with confidence while uncorrelated weights encompass uncertainty. These different behaviors of weights perform regularization and preclude over-fitting caused by bad data since uncorrelated and negative correlated weights absorb vague and outlier pattern, respectively.
\end{rem*}

%
%
\section{More Experiments}


%
%
\subsection{Regression Tasks} 
\label{subsec:reg}

We conduct three regression experiments:
1) a synthetic scenario where the training dataset
contains outliers sampled from other distributions,
2) using a Boston housing dataset with synthetic outliers,
3) a behavior cloning scenario where the
demonstrations are collected from 
both expert and adversarial policies. 

%
%
\paragraph{Synthetic Example}
We first apply ChoiceNet to a simple one-dimensional
regression problem of fitting 
$f(x) = \cos(\frac{\pi}{2}x) \exp(-(\frac{x}{2})^2)$
where $x \in [-3,+3]$ as shown in 
Figure \ref{fig:1d_reg}.
ChoiceNet is compared with 
a naive multilayer perceptron (MLP),
Gaussian process regression (GPR) \citep{Rasmussen_06},
leveraged Gaussian process regression (LGPR) with
leverage optimization \citep{SJChoi_16}, and 
robust Gaussian process regression (RGPR)
with an infinite Gaussian process mixture model 
\citep{Rasmussen_02_imgp} are also compared.
ChoiceNet has five mixtures and it has
two hidden layers with $32$ nodes with a ReLU activation function.
For the GP based methods, we use a squared-exponential
kernel function and the hyper-parameters are
determined using a simple median trick
\citep{Dai_14}\footnote{
A median trick selects the length parameter of 
a kernel function to be the median of 
all pairwise distances between training data.}. 
To evaluate its performance in corrupt datasets,
we randomly replace the original target values
with outliers whose output values are
uniformly sampled from $-1$ to $+3$.
We vary the outlier rates from $0\%$ (clean)
to $80\%$ (extremely noisy). 

Table \ref{tbl:1d_reg} illustrates the RMSEs (root mean square errors)
between the reference target function 
and the fitted results of
ChoiceNet and other compared methods.
Given an intact training dataset, all the methods
show stable performances in that the RMSEs
are all below $0.1$.
Given training datasets whose outlier rates
exceed $40\%$, however,
only ChoiceNet successfully fits
the target function
whereas the other methods fail
as shown in Figure \ref{fig:1d_reg}.

%
%
\begin{figure}
	\centering
	\includegraphics[width=.9\columnwidth]{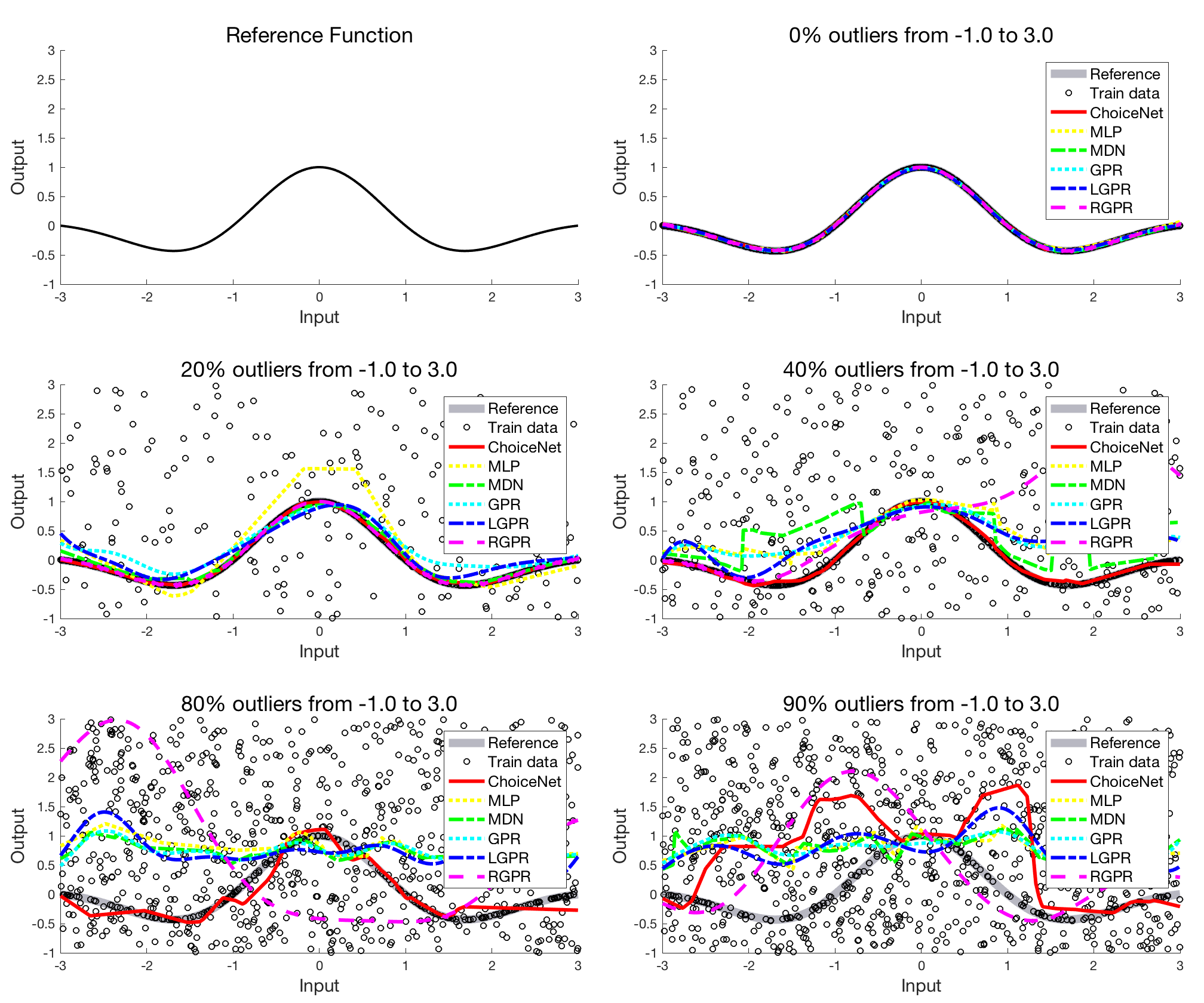}
	\caption{Reference function and
		fitting results of compared methods
		on different outlier rates,
		$0\%$,$20\%$
		$40\%$, $80\%$, and $90\%$).}
	\label{fig:1d_reg}
\end{figure}

%
%
\begin{table}[t!]  \center 
\small
\caption{The RMSEs of compared methods
	on synthetic toy examples}
\label{tbl:1d_reg}
\begin{tabular}{ l | c c c c c c  } 
    \toprule
    Outliers & ChoiceNet & MLP
                    & GPR & LGPR & RGPR & MDN\\
	\midrule
	\textit{$0\%$}		
	& $0.034$  	& $0.039$
	& $\bf{0.008}$ 	& $0.022$  & $0.017$ & $0.028$
	\\ 
	\textit{$20\%$}		
	& $0.022$ 	 	& $0.413$
	& $0.280$ 	& $0.206$ 	& $\bf{0.013}$ & $0.087$
	\\ 
	\textit{$40\%$}		
	& $\bf{0.018}$ 	 	& $0.452$
	& $0.447$ 	& $0.439$	& $1.322$ & $0.565$
	\\ 
	\textit{$60\%$}		
	& $\bf{0.023}$ 	 	& $0.636$
	& $0.602$ 	& $0.579$	& $0.738$ & $0.645$
	\\ 
	\textit{$80\%$}		
	& $\bf{0.084}$ 		& $0.829$
	& $0.779$	& $0.777$	& $1.523$ & $0.778$
	\\
	\bottomrule
\end{tabular}
\end{table}

%
%
\paragraph{Boston Housing Dataset}
\label{subsubsec:bh}

Here, we used a real world dataset, a Boston housing price dataset,
and checked the robustness of the proposed method and 
compared with standard multi-layer perceptrons with four different types of loss functions: 
standard $L2$-loss, $L1$-loss which is known to be robust to outliers, 
a robust loss (RL) function proposed in
\citep{Belagiannis_15_RobustReg}, and a leaky robust loss (LeakyRL) function. 
We further implement the leaky version of \citep{Belagiannis_15_RobustReg}
in that the original robust loss function 
with Tukey’s biweight function discards the instances 
whose residuals exceed certain threshold.

%
%
\begin{table}[t!]  \center 
\small
\caption{The RMSEs of compared methods
	on the Boston Housing Dataset}
\label{tbl:boston}
\begin{tabular}{ l | c c c c c c c } 
    \toprule
    Outliers & ChoiceNet & $L2$
                    & $L1$ & RL & LeakyRL &  MDN \\
	\midrule
	\textit{$0\%$}		
	& $3.29$  	& \bf{3.22}
	& $3.26$ 	& $4.28$  & $3.36$ & $3.46$
	\\ 
	\textit{$10\%$}		
	& \bf{3.99} 	 	& $5.97$
	& $5.72$ 	& $6.36$ 	& $5.71$ & $6.5$
	\\ 
	\textit{$20\%$}		
	& \bf{4.77} 	 	& $7.51$
	& $7.16$ 	& $8.08$	& $7.08$ & $8.62$
	\\ 
	\textit{$30\%$}		
	& \bf{5.94} 	 	& $9.04$
	& $8.65$ 	& $10.54$	& $8.67$ & $8.97$
	\\ 
	\textit{$40\%$}		
	& \bf{6.80} 		& $9.88$
	& $9.69$	& $10.94$	& $9.68$ & $10.44$
	\\
	\bottomrule
\end{tabular}
\end{table}

%
%
\begin{table*}[t!]
\center
\small
\caption{Average returns of compared methods
   on behavior cloning problems using MuJoCo}
\label{tbl:mujoco}
\begin{tabular}{l | c c c | c c c }
\toprule
\multirow{2}{*}{Outliers} & \multicolumn{3}{c|}{HalfCheetah} & \multicolumn{3}{c}{Walker2d} \\
                          & ChoiceNet  & MDN     & MLP      & ChoiceNet & MDN     & MLP    \\
\midrule
\textit{$10\%$}                   & \bf{2068.14}    & 192.53 & 852.91   & \bf{2754.08}   & 102.99 & 537.42 \\
\textit{$20\%$}                   & \bf{1498.72}    & 675.94 & 372.90   & \bf{1887.73}   & 95.29 & 1155.80 \\
\textit{$30\%$}                   & \bf{2035.91}    & 363.08 & 971.24   & -267.10   & \bf{-260.80} & -728.39 \\
\bottomrule
\end{tabular}
\end{table*}

%
%
\paragraph{Behavior Cloning Example}
\label{subsubsec:bc}

In this experiment, we apply ChoiceNet to
behavior cloning tasks when given demonstrations with mixed qualities
where the proposed method is compared with
a MLP and a MDN in two locomotion tasks:
\textit{HalfCheetah} and \textit{Walker2d}.
The network architectures are identical to
those in the synthetic regression example tasks.
To evaluate the robustness of ChoiceNet,
we collect demonstrations from both
an expert policy and an adversarial policy
where two policies are trained 
by solving the corresponding reinforcement learning problems using
the state-of-the-art proximal policy optimization (PPO)
\citep{Schulman_17proximal}.
For training adversarial policies for both tasks, we flip the signs of the 
directional rewards so that the agent gets incentivized by going backward. 
We evaluate the performances of the compared methods
using $500$ state-action pairs with different mixing ratio
and measure the average return over $100$ consecutive episodes.
The results are shown in Table \ref{tbl:mujoco}.
In both cases, ChoiceNet outperforms compared methods by a significant margin.
Additional behavior cloning experiments for autonomous driving can be found
in the supplement material.

%
%
\paragraph{Autonomous Driving Experiment}
\label{subsubsec:track2}

In this experiment, we apply ChoiceNet to
a autonomous driving scenario
in a simulated environment.
In particular, the tested methods are
asked to learn the policy 
from driving demonstrations collected from 
both safe and careless driving modes.
We use the same set of methods used
for the previous task.
The policy function is defined as a mapping 
between four dimensional input features
consist of three frontal distances
to left, center, and right lanes
and lane deviation distance from the center
of the lane to the desired heading. 
Once the desired heading is computed,
the angular velocity of a car is computed by
$10*(\theta_{\text{desired}}-\theta_{\text{current}})$
and the directional velocity is fixed to $10m/s$.
The driving demonstrations are 
collected from keyboard inputs by human users.
The objective of this experiment is to assess
its performance on a training set generated from 
two different distributions.
We would like to note that this task does not
have a reference target function in that
all demonstrations are collected manually. 
Hence, we evaluated the performances of the
compared methods by running the trained policies
on a straight track by randomly deploying 
static cars.

Table \ref{tbl:track1} and 
Table \ref{tbl:track2} indicate collision rates
and RMS lane deviation distances of the tested
methods, respectively, where the statistics are 
computed from $50$ independent runs 
on the straight lane by randomly placing
static cars as shown in Figure \ref{fig:exp_drive}.
ChoiceNet clearly outperforms compared methods
in terms of both 
safety (low collision rates) and 
stability (low RMS lane deviation distances). 

%
%
\begin{figure}[!t] 
	\centering 
	\includegraphics[width=.65\columnwidth]
		{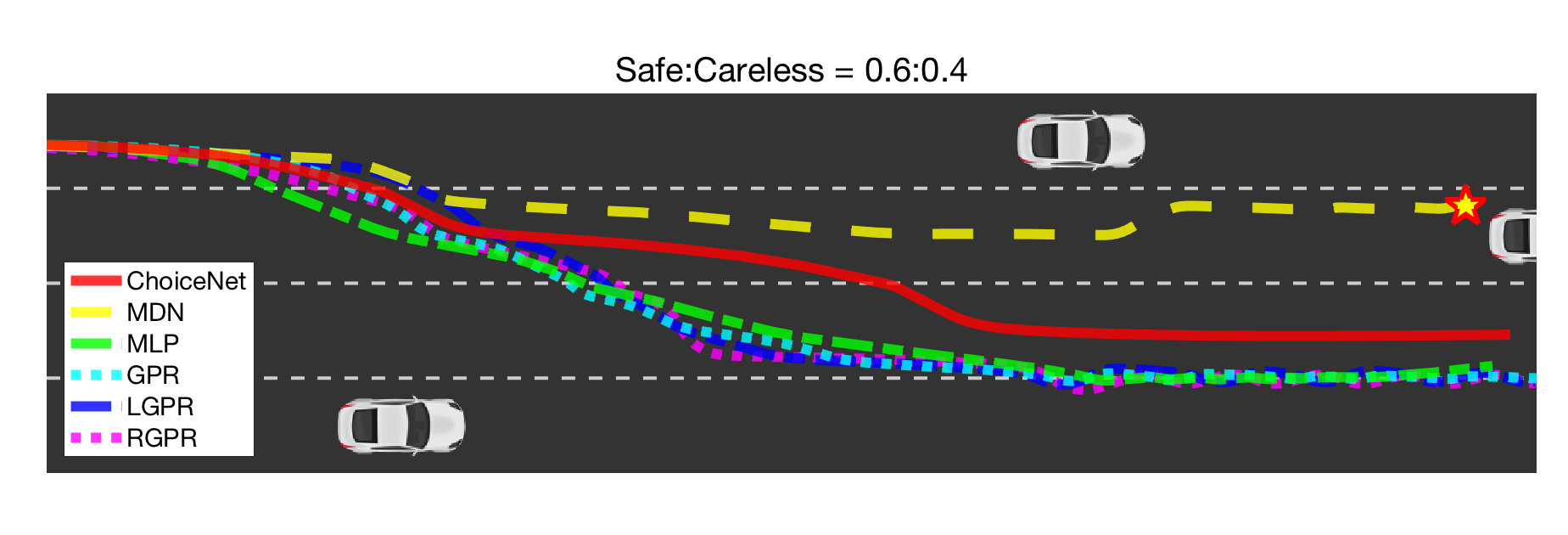}
	\caption{
		Resulting trajectories of compared methods
		trained with mixed demonstrations.
		(best viewed in color). 
		}
	\label{fig:exp_drive}
\end{figure}


%
%
\begin{table*}[t!] 
\center 
\small
\caption{Collision rates of compared methods 
on straight lanes.}
\label{tbl:track1}
\begin{tabular}{ l | c c c c c c  } 
    \toprule
    Outliers & ChoiceNet & MDN & MLP
                    & GPR & LGPR & RGPR \\
	\midrule
	\textit{$0\%$}		
	& $\bf{0\%}$ 	& $50.83\%$ 	& $\bf{0\%}$
	& $0.83\%$ & $4.17\%$  	& $3.33\%$
	\\ 
	\textit{$10\%$}	
	& $\bf{0\%}$ 	& $38.33\%$ 	& $\bf{0\%}$
	& $2.5\%$ 	& $1.67\%$ 	& $4.17\%$
	\\ 
	\textit{$20\%$}		
	& $\bf{0\%}$ 	& $41.67\%$ 	& $\bf{0\%}$
	& $7.5\%$ 	& $6.67\%$		& $10\%$
	\\ 
	\textit{$30\%$}		
	& $\bf{0\%}$ 	& $66.67\%$ 	& $1.67\%$
	& $4.17\%$ & $1.67\%$	& $7.5\%$
	\\ 
	\textit{$40\%$}		
	& $\bf{0.83\%}$ & $35\%$	& $3.33\%$
	& $6.67\%$	& $6.67\%$	& $24.17\%$
	\\
	\bottomrule
\end{tabular}
\end{table*}

\begin{table*}[t!] 
\small
\centering
\caption{Root mean square lane deviation distances (m)
of compared methods on straight lanes.}
\label{tbl:track2}
\begin{tabular}{ l | c c c c c c  } 
    \toprule
    Outliers & ChoiceNet & MDN & MLP
                    & GPR & LGPR & RGPR \\
	\midrule
	\textit{$0\%$}
	& $0.314$ & $0.723$ & $\bf{0.300}$ & $0.356$ & $0.349$ & $0.424$ 
	\\ 
	\textit{$10\%$}
	& $\bf{0.352}$ & $0.387$ & $0.438$ & $0.401$ & $0.446$ & $0.673$ 
	\\ 
	\textit{$20\%$}
	& $\bf{0.349}$ & $0.410$ & $0.513$ & $0.418$ & $0.419$ & $0.725$ 
	\\ 
	\textit{$30\%$}
	& $\bf{0.368}$ & $\bf{0.368}$ & $0.499$ & $0.455$ & $0.476$ & $0.740$ 
	\\ 
	\textit{$40\%$}
	& $\bf{0.370}$ & $0.574$ & $0.453$ & $0.453$ & $0.453$ & $0.636$ 
	\\ 
	\bottomrule
\end{tabular}
\label{tbl:track}
\end{table*}

Here, we describe the features used for the
autonomous driving experiments. 
As shown in the manuscript, we use
a four dimensional feature,
a lane deviation distance of an ego car,
and three frontal distances to the closest car
at left, center, and right lanes
as shown in 
Figure \ref{fig:track_env}. 
We upperbound the frontal distance to $40m$. 
Figure \ref{fig:exp_drive_a} and \ref{fig:exp_drive_b}
illustrate manually collected trajectories of
a safe driving mode and a careless driving mode. 

%
%
\begin{figure}[!t] 
	\centering
	\includegraphics[width=.5\columnwidth]{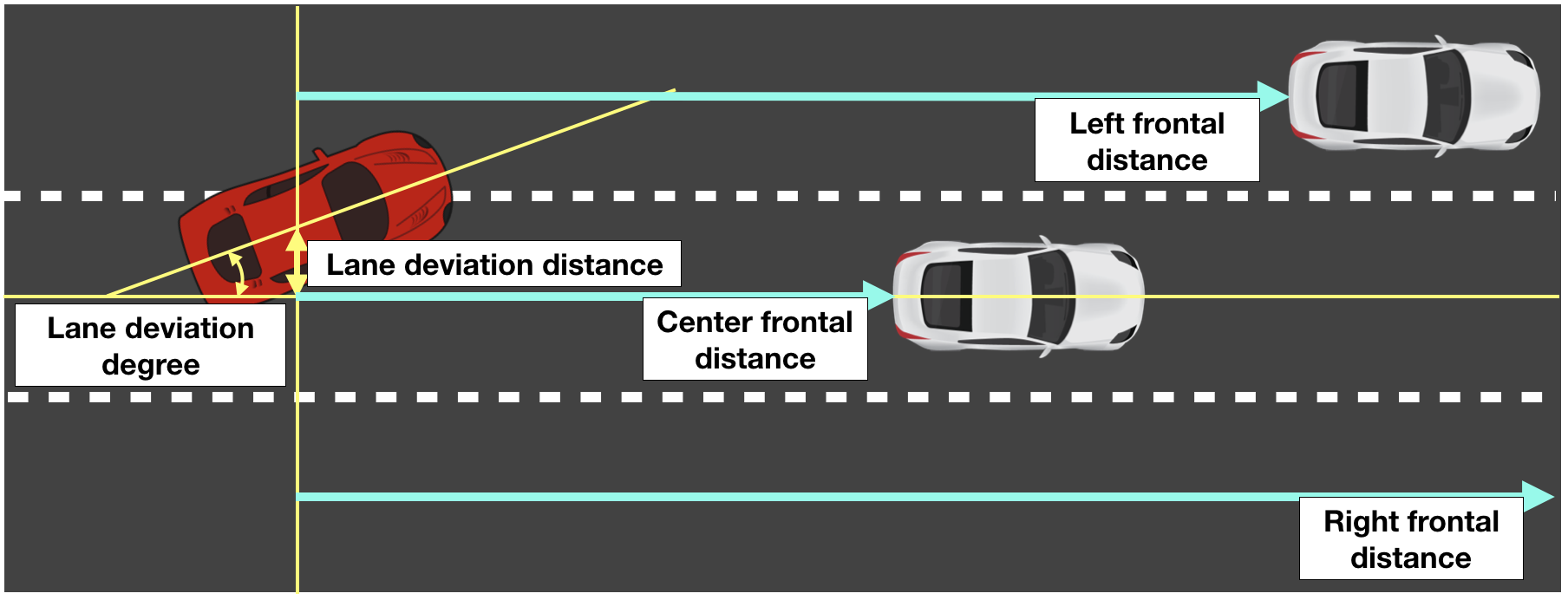}
	\caption{
	    Descriptions of the featrues of 
	    an ego red car used in
	    autonomous driving experiments.
	    }
	\label{fig:track_env}
\end{figure}

%
%
\begin{figure}[!t] 
	\centering 
	\subfigure[]{\includegraphics[width=.41\columnwidth]
		{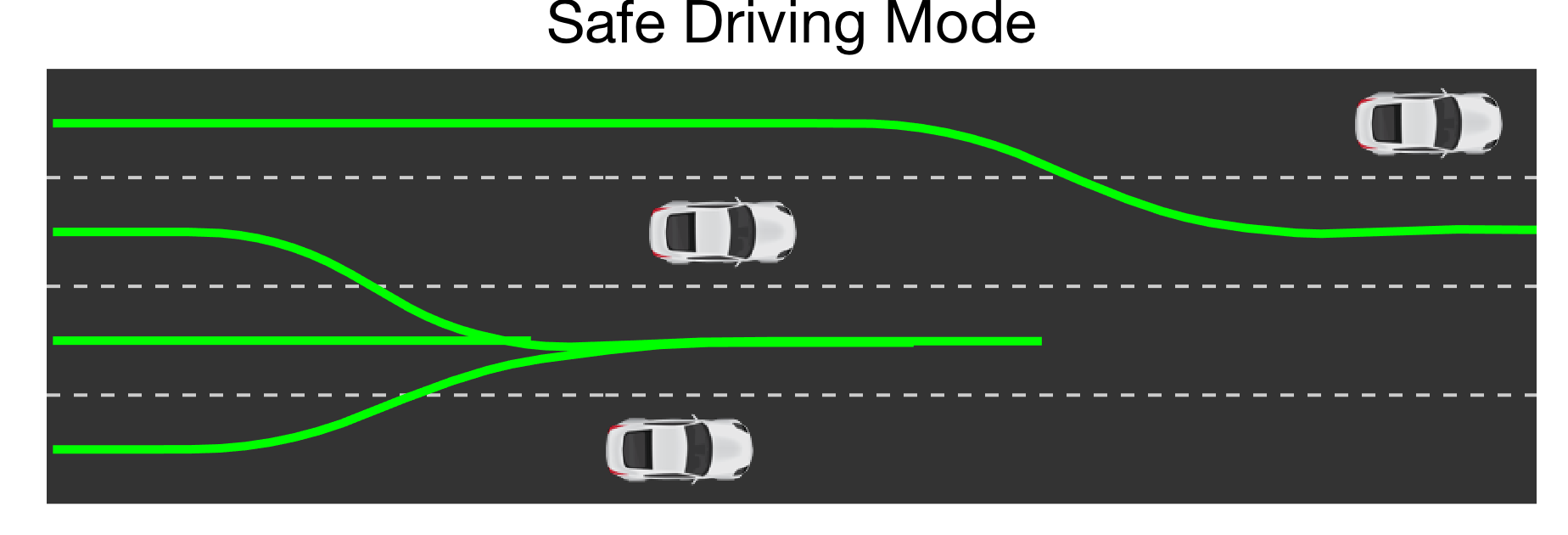}
		\label{fig:exp_drive_a}}
	\subfigure[]{\includegraphics[width=.41\columnwidth]
		{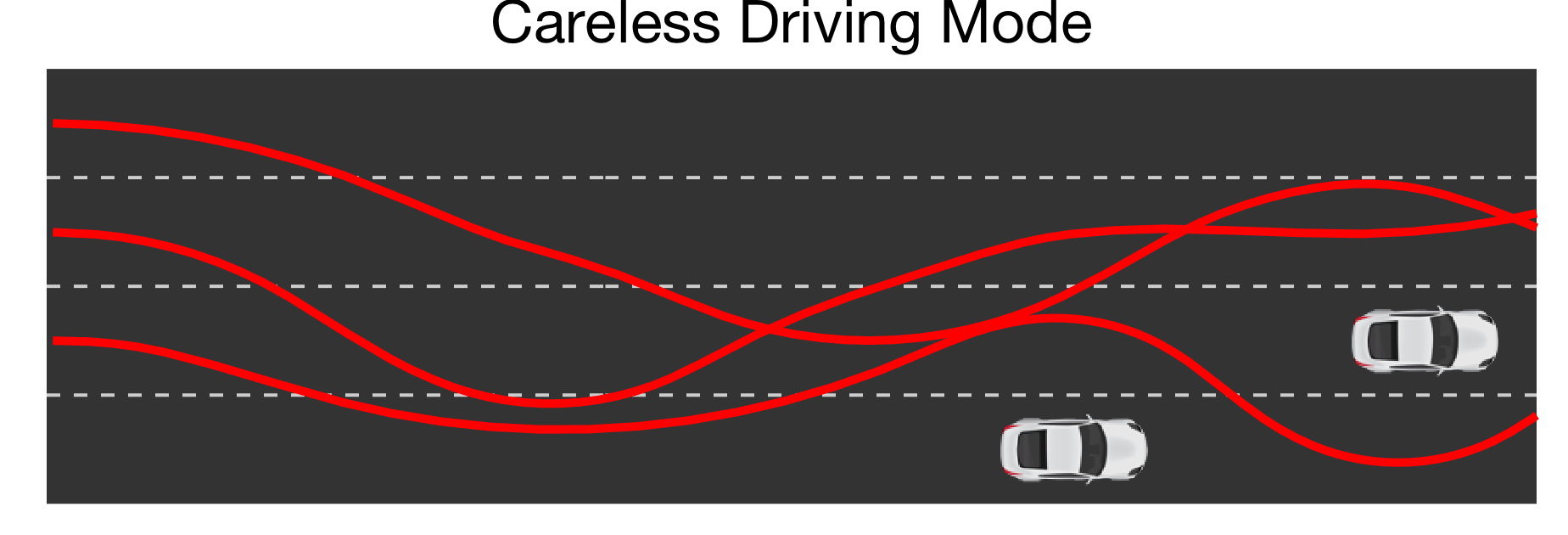}
		\label{fig:exp_drive_b}}
	\caption{
		Manually collected trajectories of
		(a) safe driving mode and (b) careless driving
		mode.
		(best viewed in color). 
		}
	\label{fig:exp_drive2}
\end{figure}

%
%
\subsection{Classification Tasks}
\label{subsec:cls}

Here, we conduct comprehensive classification experiments
using MNIST, CIFAR-10, and Large Movie Review datasets
to evaluate the performance of ChoiceNet
on corrupt labels.
For the image datasets,
we followed two different settings to generate noisy datasets:
one following the setting in \citep{Zhang_18_mixup}
and the other from \citep{Han_18}
which covers both symmetric and asymmetric noises. 
For the Large Movie Review dataset, we simply shuffle
the labels to the other in that it only contains two labels. 

%
%
\paragraph{MNIST}

For MNIST experiments, we
randomly shuffle a percentage of the labels
with the corruption probability $p$
from $50\%$ to $95\%$ and compare median accuracies 
after five runs for each configuration
following the setting in \citep{Zhang_18_mixup}.

We construct two networks: 
a network with two residual blocks \citep{he2016deep} 
with $3 \times 3 \times 64$ convolutional layers followed 
by a fully-connected layer with $256$ output units (ConvNet)
and ChoiceNet with the same two residual blocks 
followed by the MCDN block (ConvNet+CN). 
We train each network for $50$ epochs with 
a fixed learning rate of $\text{1e-5}$.

We train each network for $300$ epochs with 
a minibatch size of $256$. We begin with 
a learning rate of $0.1$, and 
it decays by $1/10$ after $150$ and $225$ epochs.
We apply random horizontal flip and random 
crop with 4$-$pixel-padding 
and use a weight decay of $0.0001$ for
the baseline network as \citep{he2016deep}. 

\begin{table}[t]
    \small
    \caption{
    Test accuracies on the MNIST datasets with
    corrupt labels.
    }
    \label{tab:minst_rs}
    \centering
    \begin{tabular}{llll}
        \toprule
        Corruption $p$        & Configuration   & Best & Last \\
        \midrule
        \multirow{3}{*}{50\%} & ConvNet         & 95.4 & 89.5 \\
                              & ConvNet+Mixup   & 97.2 & 96.8 \\
                              & ConvNet+CN       & \bf{99.2} & \bf{99.2} \\
                              & MDN & 97.7 & 97.7 \\
        \midrule
        \multirow{3}{*}{80\%} & ConvNet         & 86.3 & 76.9 \\
                              & ConvNet+Mixup   & 87.2 & 87.2 \\
                              & ConvNet+CN       & \bf{98.2} & \bf{97.6} \\
                              & MDN & 85.2 & 78.7 \\
        \midrule
        \multirow{3}{*}{90\%} & ConvNet         & 76.1 & 69.8 \\
                              & ConvNet+Mixup   & 74.7 & 74.7 \\
                              & ConvNet+CN       & \bf{94.7} & \bf{89.0} \\
                              & MDN & 61.4 & 50.2 \\
        \midrule
        \multirow{3}{*}{95\%} & ConvNet         & 72.5 & 64.4 \\
                              & ConvNet+Mixup   & 69.2 & 68.2 \\
                              & ConvNet+CN       & \bf{88.5} & \bf{80.0} \\
                              & MDN & 31.2 & 25.9 \\
        \bottomrule
    \end{tabular}
\end{table}

The classification results are shown in Table \ref{tab:minst_rs}
where ChoiceNet consistently outperforms ConvNet and ConvNet+Mixup
by a significant margin,
and the difference between the accuracies of ChoiceNet 
and the others becomes more clear as 
the corruption probability increases.

Here, we also present additional experimental results 
using the MNIST dataset on following three different 
scenarios:
\begin{enumerate}
    \item Biased label experiments
        where we randomly assign the percentage of 
        the training labels to label $0$.
    \item Random shuffle experiments
        where we randomly replace the percentage
        of the training labels from the 
        uniform multinomial distribution. 
    \item Random permutation experiments
        where we replace the percentage of
        the labels based on 
        the label permutation matrix
        where we follow the random permutation
        in \citep{Reed_14_bootstrap}.
\end{enumerate}

The best and final accuracies on the intact test dataset
for biased label experiments are shown in
Table \ref{tab:minst_rb}. 
In all corruption rates, ChoiceNet achieves the best
performance compared to two baseline methods.
The learning curves of the biased label experiments
are depicted in Figure \ref{fig:rb}.
Particularly, 
we observe unstable learning curves regarding
the test accuracies of ConvNet and Mixup. 
As training accuracies of such methods
show stable learning behaviors,
this can be interpreted as
the networks are simply memorizing noisy labels. 
In the contrary, the learning curves of ChoiceNet show 
stable behaviors which clearly indicates
the robustness of the proposed method.

The experimental results and learning curves
of the random shuffle experiments
are shown in Table \ref{tab:minst_rs2}
and Figure \ref{fig:rs}.
The convolutional neural networks trained with Mixup show
robust learning behaviors when $80\%$
of the training labels are uniformly shuffled.
However, given an extremely noisy dataset ($90\%$
and $95\%$),
the test accuracies of baseline methods 
decrease as the number of epochs increases.
ChoiceNet shows outstanding robustness
to the noisy dataset
in that the test accuracies do not drop even after
$50$ epochs for the cases where the corruption rates
are below $90\%$.
For the $95\%$ case, however, over-fitting is occured in
all methods.

Table \ref{tab:minst_rp} and Figure \ref{fig:rp}
illustrate the results of the random permutation 
experiments.
Specifically, we change the labels of
randomly selected training data using
a permutation rule:
$(0,1,2,3,4,5,6,7,8,9) \to (7, 9, 0, 4, 2, 1, 3, 5, 6, 8)$
following \citep{Reed_14_bootstrap}.
We argue that this setting is more arduous than 
the random shuffle case in that we are intentionally 
changing the labels based on predefined permutation rules.

%
%
\begin{table}[!t]
    \small
    \caption{Test accuracies on the MNIST dataset with
        biased label.}
    \label{tab:minst_rb}
    \centering
    \smallskip
    \begin{tabular}{llll}
        \toprule
        Corruption $p$        & Configuration   & Best & Last \\
        \midrule
        \multirow{3}{*}{25\%} & ConvNet         & 95.4 & 89.5 \\
                              & ConvNet+Mixup   & 97.2 & 96.8 \\
                              & ChoiceNet       & \bf{99.2} & \bf{99.2} \\
        \midrule
        \multirow{3}{*}{40\%} & ConvNet         & 86.3 & 76.9 \\
                              & ConvNet+Mixup   & 87.2 & 87.2 \\
                              & ChoiceNet       & \bf{98.2} & \bf{97.6} \\
        \midrule
        \multirow{3}{*}{45\%} & ConvNet         & 76.1 & 69.8 \\
                              & ConvNet+Mixup   & 74.7 & 74.7 \\
                              & ChoiceNet       & \bf{94.7} & \bf{89.0} \\
        \midrule
        \multirow{3}{*}{47\%} & ConvNet         & 72.5 & 64.4 \\
                              & ConvNet+Mixup   & 69.2 & 68.2 \\
                              & ChoiceNet       & \bf{88.5} & \bf{80.0} \\
        \bottomrule
    \end{tabular}
\end{table}

%
%
\begin{table}[!t]
    \small
    \caption{Test accuracies on the MNIST dataset with
        corrupt label.}
    \label{tab:minst_rs2}
    \centering
    \smallskip
    \begin{tabular}{llll}
        \toprule
        Corruption $p$        & Configuration   & Best & Last \\
        \midrule
        \multirow{3}{*}{50\%} & ConvNet         & 97.1 & 95.9 \\
                              & ConvNet+Mixup   & 98.0 & 97.8 \\
                              & ChoiceNet       & \bf{99.1} & \bf{99.0} \\
        \midrule
        \multirow{3}{*}{80\%} & ConvNet         & 90.6 & 79.0 \\
                              & ConvNet+Mixup   & 95.3 & 95.1 \\
                              & ChoiceNet       & \bf{98.3} & \bf{98.3} \\
        \midrule
        \multirow{3}{*}{90\%} & ConvNet         & 76.1 & 54.1 \\
                              & ConvNet+Mixup   & 78.6 & 42.4 \\
                              & ChoiceNet       & \bf{95.9} & \bf{95.2} \\
        \midrule
        \multirow{3}{*}{95\%} & ConvNet         & 50.2 & 31.3 \\
                              & ConvNet+Mixup   & 53.2 & 26.6 \\
                              & ChoiceNet       & \bf{84.5} & \bf{66.0} \\
        \bottomrule
    \end{tabular}
\end{table}

%
%
\begin{table}[!t]
    \small
    \caption{Test accuracies on the MNIST dataset with
        randomly permutated label.}
    \label{tab:minst_rp}
    \centering
    \smallskip
    \begin{tabular}{llll}
        \toprule
        Corruption $p$        & Configuration   & Best & Last \\
        \midrule
        \multirow{3}{*}{25\%} & ConvNet         & 94.4 & 92.2 \\
                              & ConvNet+Mixup   & 97.6 & 97.6 \\
                              & ChoiceNet       & \bf{99.2} & \bf{99.2} \\
        \midrule
        \multirow{3}{*}{40\%} & ConvNet         & 77.9 & 71.8 \\
                              & ConvNet+Mixup   & 84.0 & 83.0 \\
                              & ChoiceNet       & \bf{99.2} & \bf{98.8} \\
        \midrule
        \multirow{3}{*}{45\%} & ConvNet         & 68.0 & 61.4 \\
                              & ConvNet+Mixup   & 68.9 & 55.8 \\
                              & ChoiceNet       & \bf{98.0} & \bf{97.1} \\
        \midrule
        \multirow{3}{*}{47\%} & ConvNet         & 58.2 & 53.9 \\
                              & ConvNet+Mixup   & 60.2 & 53.4 \\
                              & ChoiceNet       & \bf{92.5} & \bf{86.1} \\
        \bottomrule
    \end{tabular}
\end{table}

%
%
\begin{figure}[!t] 
	\centering 
	\subfigure[]{\includegraphics[width=.44\columnwidth]
		{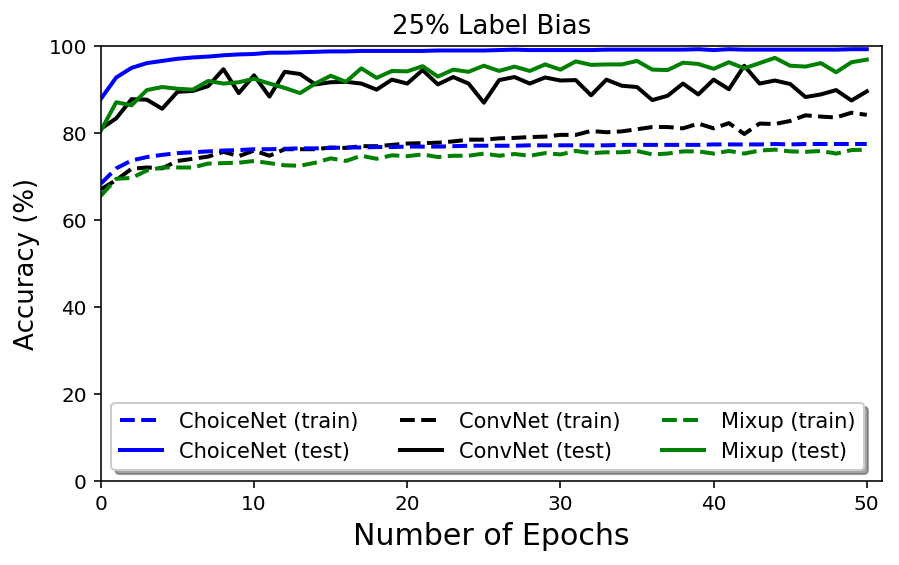}
		\label{fig:rb_a}}
	\subfigure[]{\includegraphics[width=.44\columnwidth]
		{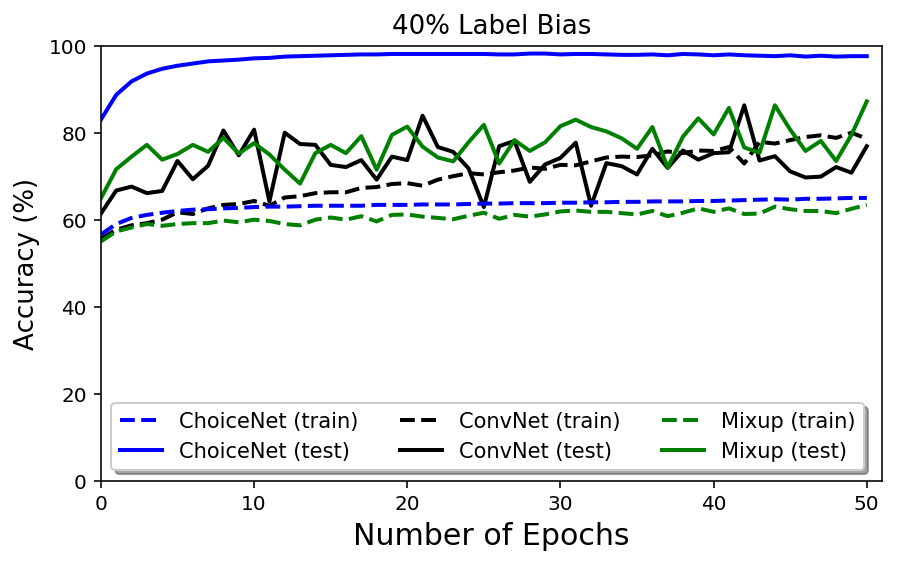}
		\label{fig:rb_b}}
	\subfigure[]{\includegraphics[width=.44\columnwidth]
		{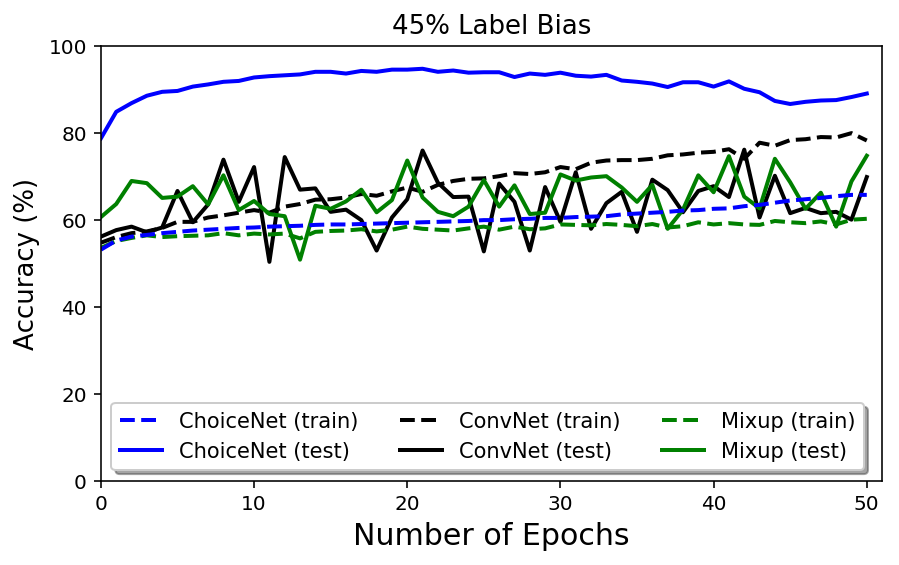}
		\label{fig:rb_c}}
	\subfigure[]{\includegraphics[width=.44\columnwidth]
		{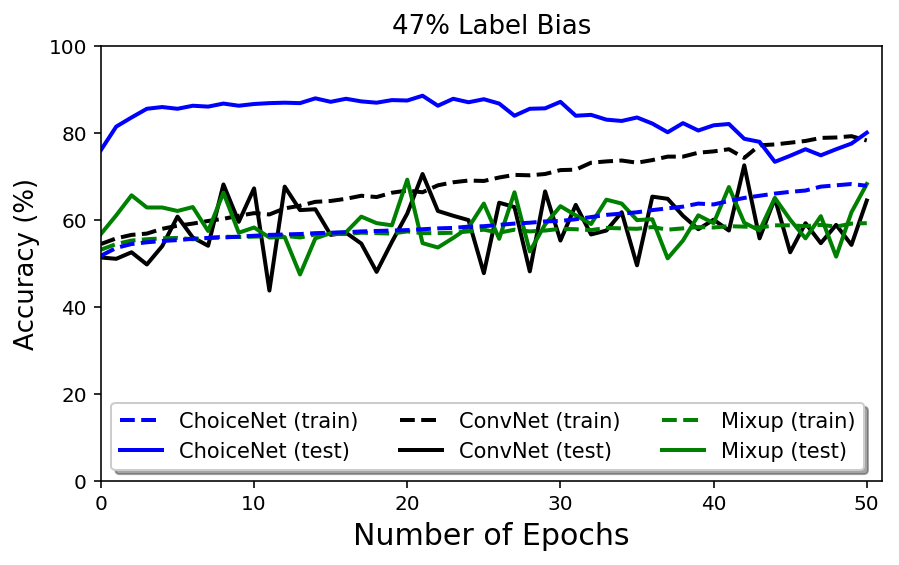}
		\label{fig:rb_d}}
	\caption{
	    Learning curves of compared methods on 
	    random bias experiments using MNIST with 
	    different noise levels. 
		}
	\label{fig:rb}
\end{figure}

%
%
\begin{figure}[!t] 
	\centering 
	\subfigure[]{\includegraphics[width=.44\columnwidth]
		{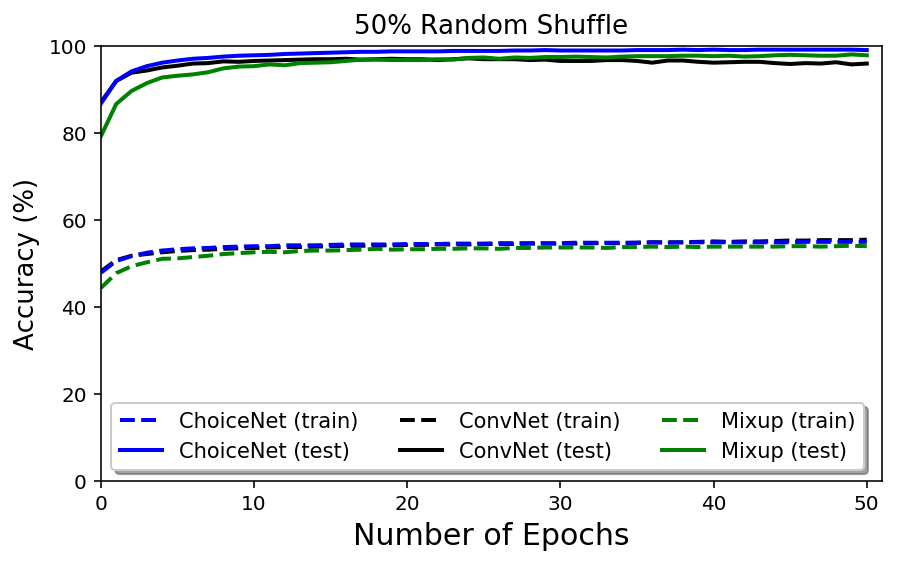}
		\label{fig:rs_a}}
	\subfigure[]{\includegraphics[width=.44\columnwidth]
		{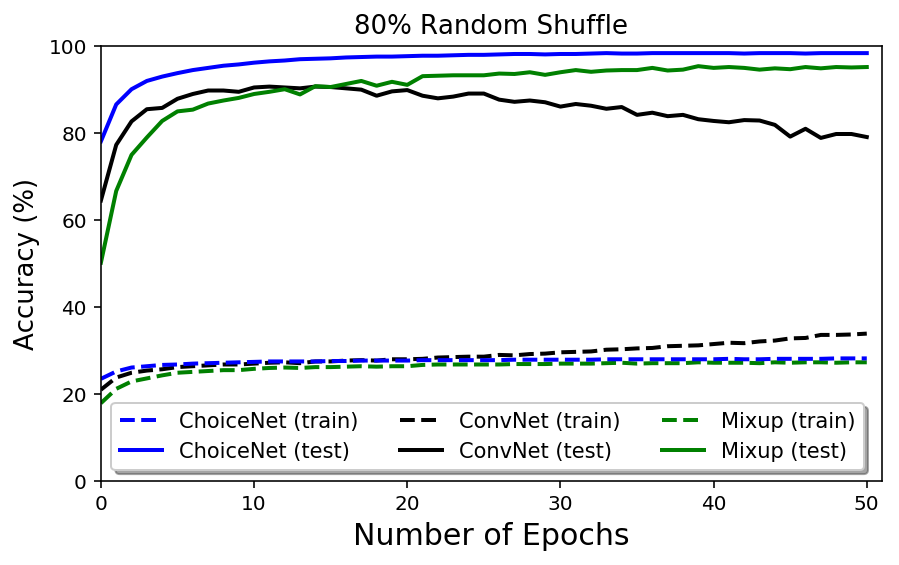}
		\label{fig:rs_b}}
	\subfigure[]{\includegraphics[width=.44\columnwidth]
		{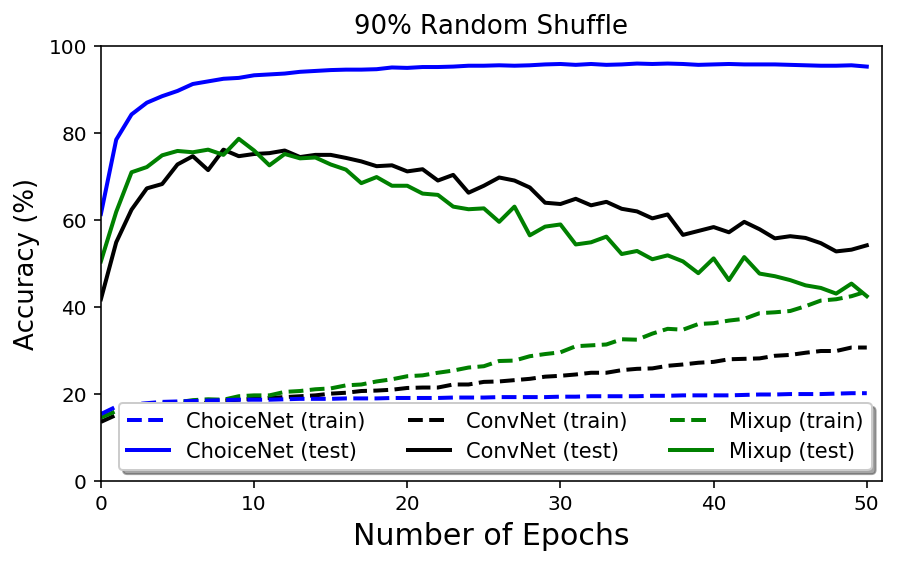}
		\label{fig:rs_c}}
	\subfigure[]{\includegraphics[width=.44\columnwidth]
		{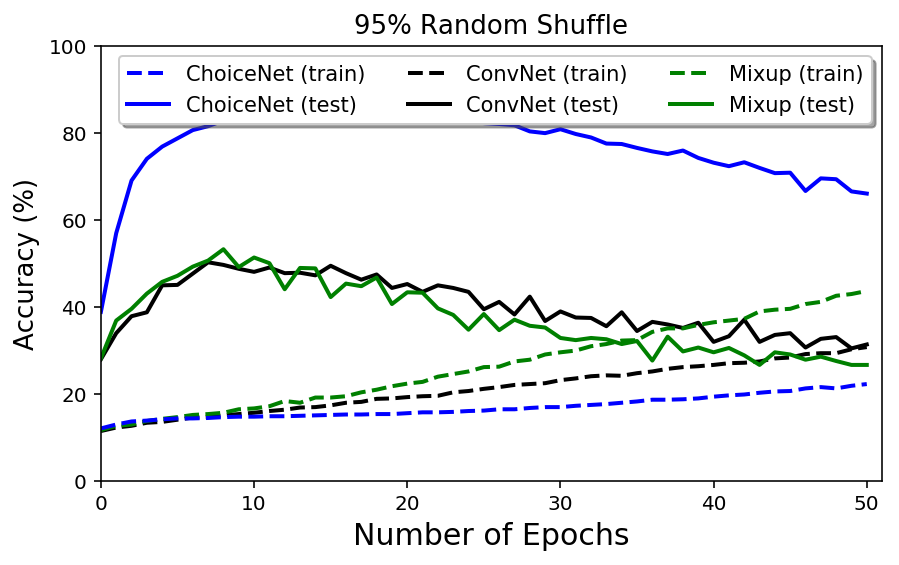}
		\label{fig:rs_d}}
	\caption{
	    Learning curves of compared methods on 
	    random shuffle experiments using MNIST with 
	    different noise levels. 
		}
	\label{fig:rs}
\end{figure}

%
%
\begin{figure}[!t] 
	\centering 
	\subfigure[]{\includegraphics[width=.42\columnwidth]
		{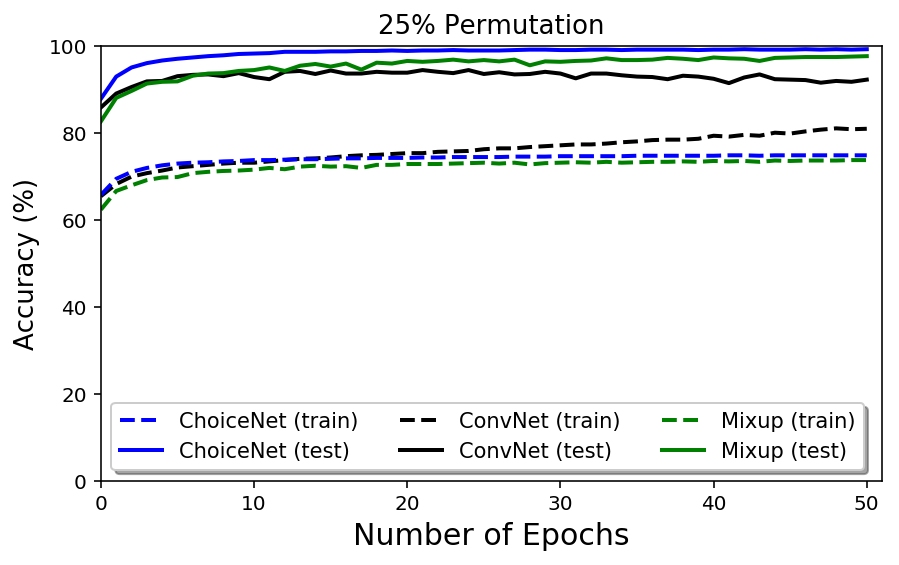}
		\label{fig:rp_a}}
	\subfigure[]{\includegraphics[width=.42\columnwidth]
		{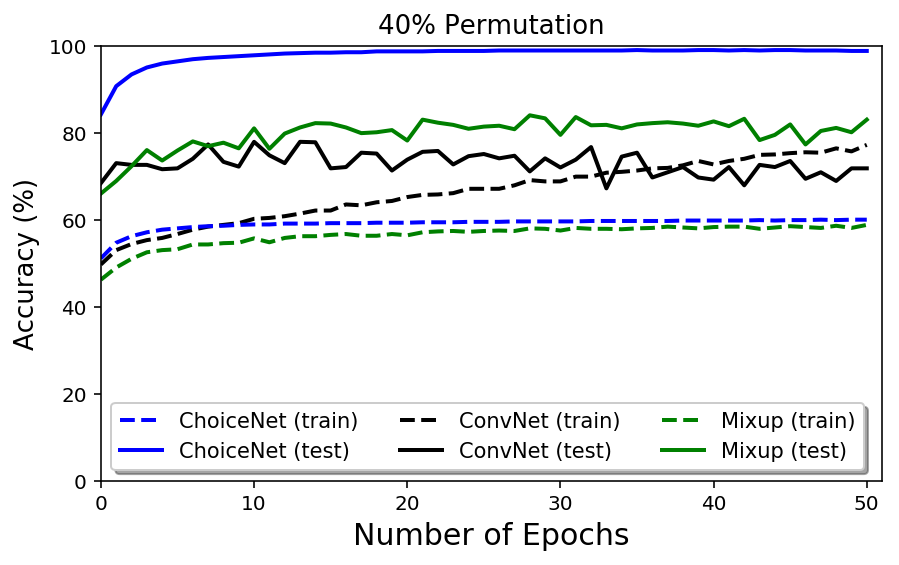}
		\label{fig:rp_b}}
	\subfigure[]{\includegraphics[width=.42\columnwidth]
		{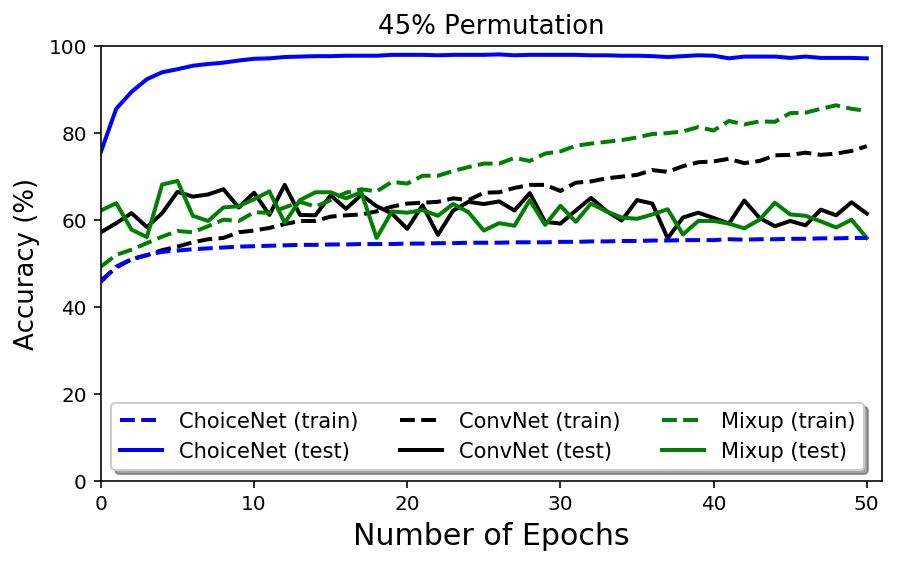}
		\label{fig:rp_c}}
	\subfigure[]{\includegraphics[width=.42\columnwidth]
		{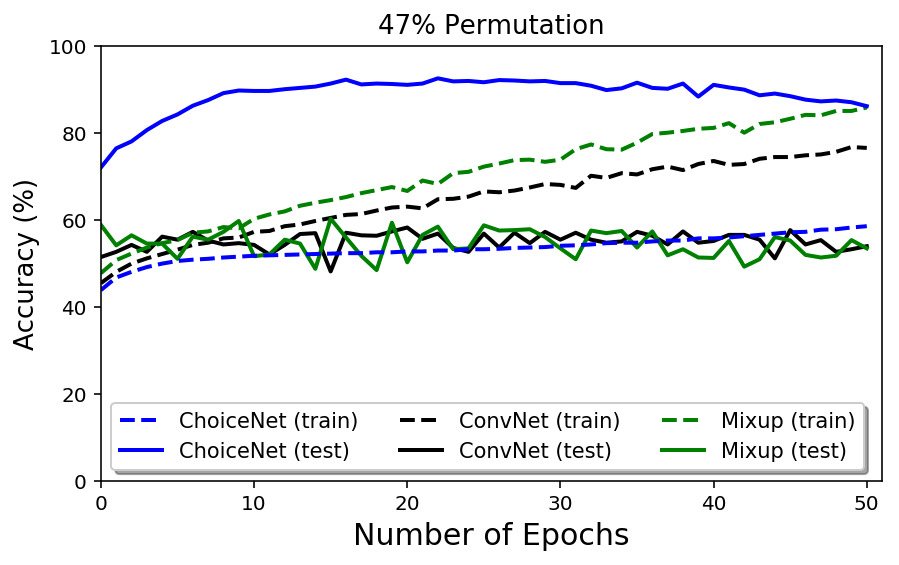}
		\label{fig:rp_d}}
	\caption{
	    Learning curves of compared methods on 
	    random permutation experiments using MNIST with 
	    different noise levels. 
		}
	\label{fig:rp}
\end{figure}

%
%
\paragraph{CIFAR-10}

When using the CIFAR-10 dataset, 
we followed two different settings from
\citep{Zhang_18_mixup} and \citep{Han_18} 
for more comprehensive comparisons.
Note that the setting in \citep{Han_18} incorporates
both symmetric and asymmetric noises.

For the first scenario following \citep{Zhang_18_mixup},
we apply symmetric noises to the labels
and vary the corruption probabilities from $50\%$ to $80\%$.
We compare our method with
Mixup \citep{Zhang_18_mixup}, VAT \citep{Miyato18}, 
and MentorNet \citep{Jiang_17_mentornet}.\footnote{
We use the authors’ implementations available online.}
We adopt WideResNet (WRN) \citep{Zagoruyko2016WRN} with
$22$ layers and a widening factor of $4$.
To construct ChoiceNet,
we replace the last layer of WideResNet with a MCDN block. 
We set $K = 3$, $\rho_\text{max} = 0.95$,
$\lambda_\text{reg} = 0.0001$,
and $\rho_k, \pi_k, \Sigma_0$ modules consist
of two fully connected layers 
with $64$ hidden units and a ReLU activation function.
We train each network for $300$ epochs with 
a minibatch size of $256$. We begin with 
a learning rate of $0.1$, and 
it decays by $1/10$ after $150$ and $225$ epochs.
We apply random horizontal flip and random 
crop with 4$-$pixel-padding 
and use a weight decay of $0.0001$ for 
the baseline network as \citep{he2016deep}. 
To train ChoiceNet, we set the weight decay rate to 
$1e-6$ and apply gradient clipping at $1.0$.

%
%
\begin{table}[t]
    \small
    \caption{
    Test accuracies on the CIFAR-10 datasets
    with symmetric noises. 
    }
    \label{tab:cifar}
    \centering
    \begin{tabular}{l l l c}
        \toprule
        Corruption $p$      & Configuration      & Accuarcy ($\%$) \\
        \midrule
        \multirow{7}{*}{50\%} & ConvNet         & 59.3 \\
                              & ConvNet+CN           & 84.6 \\
                              & ConvNet+Mixup  & 83.1 \\
                              & ConvNet+Mixup+CN   & \textbf{87.9} \\
                              & MentorNet  & 49.0  \\
                              & VAT  & 71.6  \\
                              & MDN &  58.6 \\
        \midrule
        \multirow{7}{*}{80\%} & ConvNet          & 27.4 \\
                              & ConvNet+CN          & 65.2 \\
                              & ConvNet+Mixup  & 62.9 \\
                              & ConvNet+Mixup+CN  & \textbf{75.4} \\
                              & MentorNet  & 21.4  \\
                              & VAT & 16.9  \\       
                              & MDN &  22.7 \\                                     
        \bottomrule
    \end{tabular}
\end{table}

Table \ref{tab:cifar} shows the test accuracies of compared methods
under different symmetric corruptions probabilities.
In all cases, ConvNet+CN outperforms the compared methods. 
We would like to emphasize that
when ChoiceNet and Mixup \citep{Zhang_18_mixup} are combined,
it achieves a high accuracy of $75\%$ even on the $80\%$ shuffled dataset.
We also note that ChoiceNet (without Mixup) outperforms WideResNet+Mixup 
when the corruption ratio is over $50\%$
on the last accuracies.

We conduct additional experiments on the CIFAR-10 dataset 
to better evaluate the performance on both both symmetric and 
asymmetric noises following \citep{Han_18}:
Pair-$45\%$, Symmetry-$50\%$, and Symmetry-$20\%$.
Pair-$45\%$ flips $45\%$ of each label to the next label, e.g.,
randomly flipping $45\%$ of label $1$ to label $2$ and label $2$ to label $3$, and
Symmetry-$50\%$ randomly assigns $50\%$ of each label to other labels uniformly.
We implement the a $9$-layer CNN architecture following VAT \citep{Miyato18}
and Co-teaching \citep{Han_18} for fair and accurate evaluations
and  set other configurations such as the network topology
and activations to be the same as \citep{Han_18}.
We also copied some results in \citep{Han_18} for better comparisons.
Here, we compared our method with MentorNet \citep{Jiang_17_mentornet},
Co-teaching \citep{Han_18},
and F-correction \citep{Patrini_17_LossCorrection}.

\begin{table}[t]
	\small
	\caption{
		Test accuracies on the CIFAR-10 dataset with by symmetric and asymmetric noises. 
	}
    \label{tab:cifar2}
    \centering
	\begin{tabular}{ l  l  l  l  }
	\toprule
      		       & Pair-$45\%$ & sym-$50\%$ & sym-$20\%$ \\ 
      		       \midrule
	ChoiceNet    & 70.3        & \bf{85.2}       & \bf{91.0}       \\ 
	MentorNet    & 58.14       & 71.10      & 80.76      \\ 
	Co-teaching  & \bf{72.62}       & 74.02      & 82.32      \\ 
	F-correction & 6.61        & 59.83      & 59.83      \\ 
	MDN          & 51.4        & 58.6      & 81.4      \\
	\bottomrule
\end{tabular}
\end{table}

While our proposed method outperforms all compared methods on the
symmetric noise settings, it shows the second best performance on 
asymmetric noise settings (Pair-$45\%$).
This shows the weakness of the proposed method. 
In other words, as Pair-$45\%$ assigns $45\%$ of each label to its next label, 
the MCDN fails to correctly infer the dominant label distributions.
However, we would like to note that Co-teaching \citep{Han_18}
is complementary to our method
where one can combine these two methods by using two ChoiceNets 
and update each network using Co-teaching. 
However, it is outside the scope of this paper.

Here, we also present detailed learning curves of the CIFAR-10 experiments
while varying the noise level from $20\%$ to $80\%$ following
the configurations in \citep{Zhang_18_mixup}.

%
%
\begin{figure}[!h] 
	\centering 
	\subfigure[]{\includegraphics[width=.42\columnwidth]
		{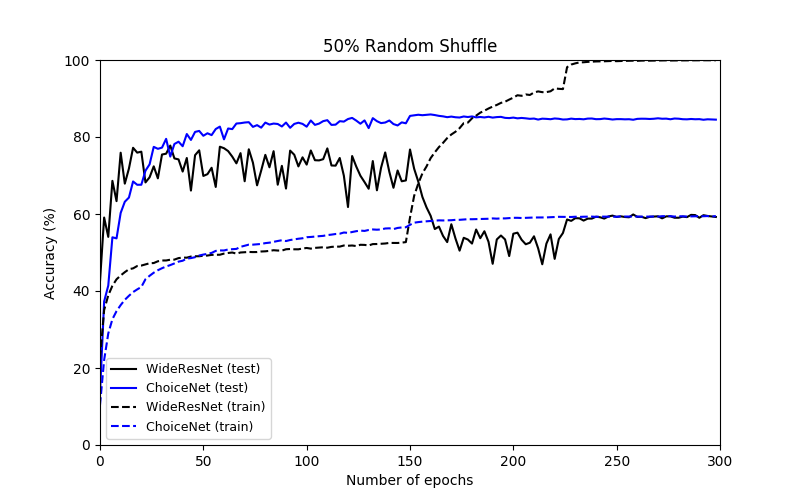}
		\label{fig:cifar_rs50}}
	\subfigure[]{\includegraphics[width=.42\columnwidth]
		{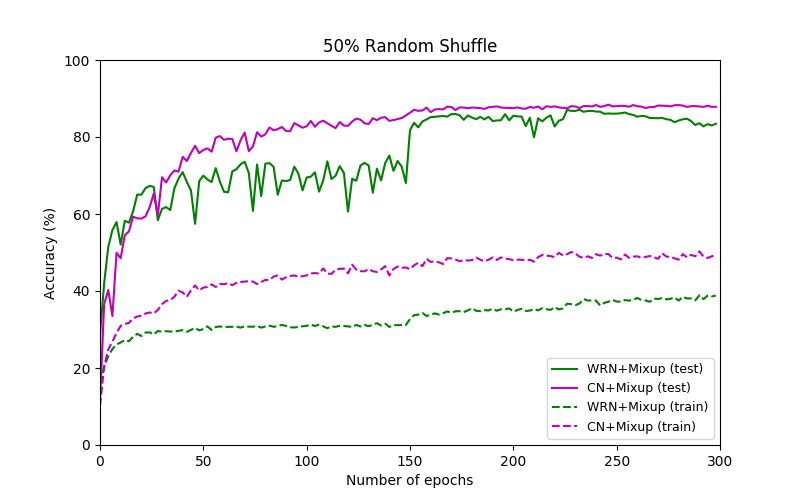}
		\label{fig:cifar_rs50_mixup}}
	\subfigure[]{\includegraphics[width=.42\columnwidth]
		{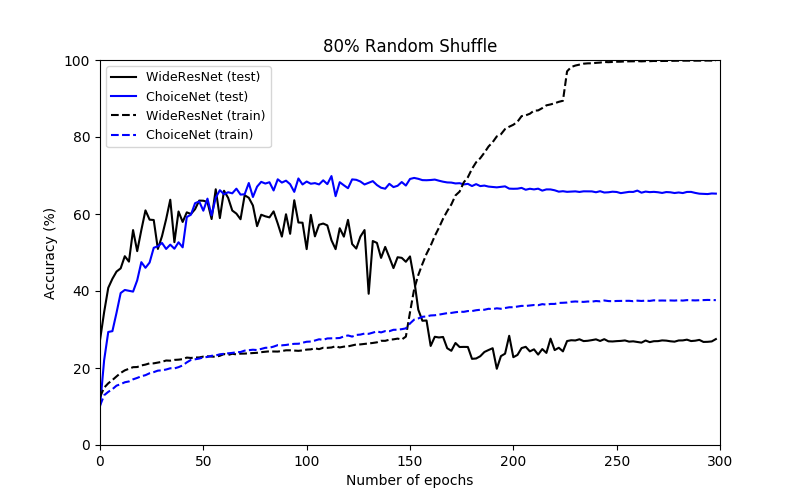}
		\label{fig:cifar_rs80}}
	\subfigure[]{\includegraphics[width=.42\columnwidth]
		{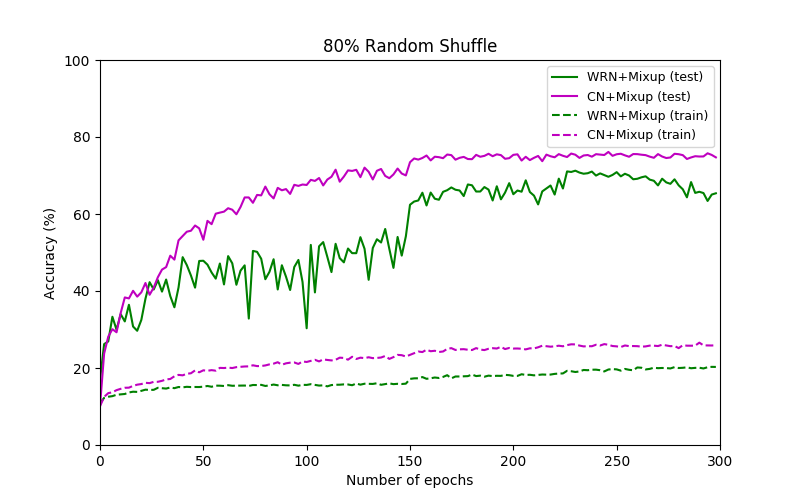}
		\label{fig:cifar_rs80_mixup}}
	\caption{
	    Learning curves of compared methods on 
	    CIFAR-10 experiments with 
	    different noise levels. 
		}
	\label{fig:cifar}
\end{figure}

%
%
\paragraph{Large Movie Review Dataset}
We also conduct a natural language processing task using a Large Movie Review dataset
which consists of $25,000$ movie reviews for training and $25,000$ reviews for testing. 
Each movie review (sentences) is mapped to a $128$-dimensional embedding vector using 
a feed-forward Neural-Net Language Models \citep{Bengio_03}.
We evaluated the robustness of the proposed method with Mixup \citep{Zhang_18_mixup}, 
VAT \cite{Miyato18}, and naive MLP baseline by randomly flipping the labels with
a corruption probability $p$.
In all experiments, we used two hidden layers with $128$ units and ReLU activations. 
The test accuracies of the compared methods are shown in Table \ref{tbl:nlp}. 
ChoiceNet shows the superior performance in the presence of outliers
where we observe that the proposed method can be used for NLP tasks as well. 

%
%
\begin{table}[]
	\small
	\caption{
		Test accuracies on the Large Movie Review dataset with different corruption probabilities. 
	}
	\label{tbl:nlp}
    \centering
	\begin{tabular}{l l l l l l}
	\toprule
	  Corruption $p$        & $0\%$     & $10\%$    & $20\%$    & $30\%$    & $40\%$    \\
	          \midrule
	ChoiceNet & 79.43 & \bf{79.50} & \bf{78.66} & \bf{77.1}  & \bf{73.98} \\
	Mixup     & \bf{79.77} & 78.73 & 77.58 & 75.85 & 69.63 \\
	MLP       & 79.04 & 77.88 & 75.70 & 69.05 & 62.83 \\
	VAT       & 76.40 & 72.50 & 69.20 & 65.20 & 58.30 \\
	\bottomrule
	\end{tabular}
\end{table}

%
%
\subsection{Ablation Study on MNIST} 
\label{subsubsec:cls2}

\begin{center}
	\includegraphics[width=.95\columnwidth]{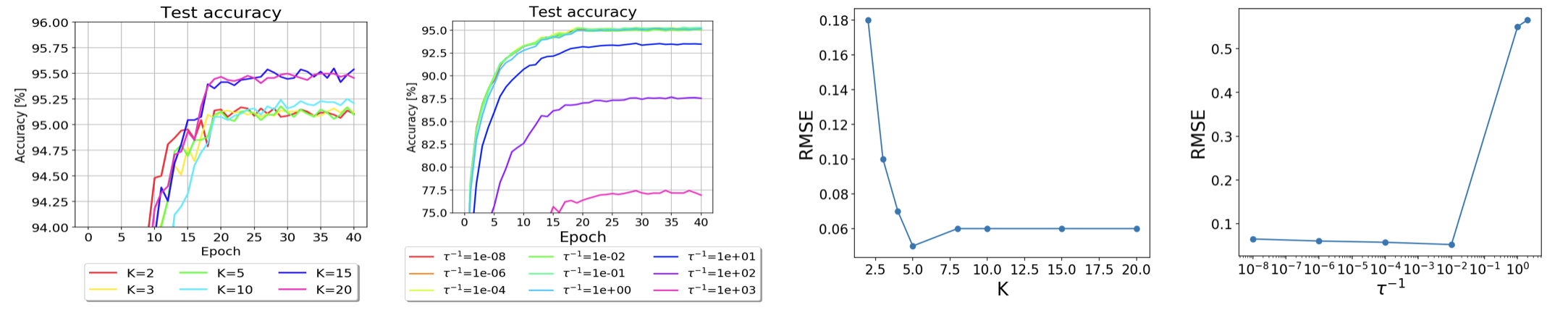}
\par\end{center}
Above figures show the results of ablation study
when varying the number of mixture $K$
and the expected measurement variance $\tau^{-1}$.
Left two figures indicate test accuracies using the MNIST dataset
where $90\%$ of train labels are randomly shuffled and 
right two figures are RMSEs using a synthetic one-dimensional 
regression problem in Section 4.1. 
We observe that having bigger $K$ is beneficial to the classification 
accuracies. In fact, the results achieved here with $K$ equals $15$
and $20$ are better than the ones reported in the submitted manuscript. 
$\tau^{-1}$ does not affect much unless it is exceedingly large.

\end{document}